\documentclass[a4paper,11pt,fleqn]{article}
\usepackage{pstricks}
\usepackage{amsmath}
\usepackage{amssymb}
\usepackage{mathrsfs} 
\usepackage{theorem} 
\usepackage{booktabs}
\usepackage{euscript}
\usepackage{exscale,relsize}
\usepackage{booktabs}
\usepackage{srcltx}
\usepackage{xcolor}
\usepackage{url}
\newcommand{\email}[1]{\href{mailto:#1}{\nolinkurl{#1}}}
\usepackage{amssymb}
\usepackage{amsmath}
\input{tcilatex}
\usepackage{times}

\usepackage{amssymb}
\usepackage{mathrsfs} 

\title{\sffamily\LARGE  \bf Empirical bounds for functions with weak interactions}

\author{Andreas Maurer \\
\small Adalbertstrasse 55, D-80799 \\
\small Munchen, Germany \\ \\
Massimiliano Pontil \\
\small Istituto Italiano di Tecnologia,
16163 Genoa, Italy \\
\small and \\
\small University College London, London WC1E 6BT, UK 
}
\usepackage[sort&compress,square,comma,authoryear]{natbib}
\begin{document}

\maketitle

\begin{abstract}
We provide sharp empirical estimates of expectation, variance and normal
approximation for a class of statistics whose variation in any argument does
not change too much when another argument is modified. Examples of such weak
interactions are furnished by U- and V-statistics, Lipschitz L-statistics
and various error functionals of $\ell_2$-regularized algorithms and Gibbs
algorithms.
\end{abstract}

\section{Introduction\label{Section Introduction}}

A central problem of learning is to relate a finite number of observations
to some underlying law. If the law is not deterministic, the appropriate
model is a sequence of random variables $X_{i}$ taking
values in some space $\mathcal{X}$. Under the idealizing assumption of
noninterfering observations of identically prepared systems, we assume these
variables to be independent and identically distributed according to some
probability measure $\mu $ on $\mathcal{X}$.%

Any quantitative model of the law based on the observations $\mathbf{X}%
=\left( X_{1},...,X_{n}\right) $ involves the computation of functions $f:%
\mathcal{X}^{n}\rightarrow 
\mathbb{R}
$. For example $f\left( \mathbf{x}\right) $ could be a bit computed by a
machine-learning program based on the training sample $\mathbf{x}$, or a
statistic to estimate some parameter like a moment, quantile or correlation
underlying the observed phenomenon. Here we will only consider bounded real
valued functions $f$.\textbf{\ }

What can we say about the expectation $E\left[ f\right] $ of $f\left( 
\mathbf{X}\right) $? What about its variance, and how can we describe the
distribution of $f\left( \mathbf{X}\right) $?

Without any assumptions on $\mu $, the answer depends on the class of
functions under consideration. Many well known and satisfactory answers
exist for the sample mean $f:\left[ 0,1\right] ^{n}\rightarrow 
\mathbb{R}
$ given by%
\begin{equation}
f\left( \mathbf{x}\right) =\frac{1}{n}\sum_{i=1}^{n}x_{i}.
\label{Sample mean}
\end{equation}%
The Chernov and Hoeffding inequalities \citep{McDiarmid 1998,Boucheron13} give high-probability estimates of $E\left[ f\right] $.
Bernstein's inequality is often stronger, but contains the variance as a
parameter of the distribution, which requires a separate estimate. Another
highlight is the Berry-Esseen theorem \citep{Berry 1941} giving rates for the
approximation of $f\left( \mathbf{X}\right) $ by an appropriately scaled
normal variable, but again both expressions for the limiting distribution
and for the approximation error contain the variance.

The variance of $X_{i}$ can be estimated by the sample variance $v:\left[ 0,1%
\right] ^{n}\rightarrow 
\mathbb{R}
$%
\begin{equation}
v_{n}\left( \mathbf{x}\right) =\frac{1}{2n\left( n-1\right) }\sum_{i,j\in
\left\{ 1,...,n\right\} :i\neq j}\left( x_{i}-x_{j}\right) ^{2},
\label{Sample variance}
\end{equation}%
and it is shown by \citet{Maurer 2009} \citep[see also][]{Audibert}
that, for any $\delta >0$, with probability at least $1-\delta $ we have 
$\left\vert \sigma \left( X_{i}\right) -\sqrt{v_{n}\left( \mathbf{X}\right) }%
\right\vert \leq \sqrt{ \frac{2}{n-1}\ln \left( 2/\delta \right) }$.
This estimate can be combined with Bernstein's inequality to give a purely
empirical estimate of expectation, an empirical Bernstein bound, which is
superior to Hoeffding's inequality for functions of small variance \citep{Audibert,Maurer 2009}. Similarly the variance estimate can also be used in results
about normal approximation.%

In this paper we extend these results to general, not necessarily additive
functions of independent random variables. Clearly the same quantitative
results cannot be expected in great generality, but there is a class of
functions whose statistical properties are in many ways very similar to
those of the sample mean, even though some of these functions may look
highly nonlinear at first glance.%

To describe this class we introduce some notation which will be used
throughout. For $k\in \left\{ 1,...,n\right\} $ and $y,y^{\prime }\in 
\mathcal{X}$ we define the partial difference operator $D_{y,y^{\prime
}}^{k} $ acting on bounded functions $f:\mathcal{X}^{n}\rightarrow 
\mathbb{R}
$ by 
\begin{equation*}
D_{y,y^{\prime }}^{k}f\left( \mathbf{x}\right) =f\left(
x_{1},...,x_{k-1},y,x_{k+1},...,x_{n}\right) -f\left(
x_{1},...,x_{k-1},y^{\prime },x_{k+1},...,x_{n}\right) \text{.} 
\end{equation*}%
Note that $D_{y,y^{\prime }}^{k}f\left( \mathbf{x}\right) $ depends on $y$
and $y^{\prime }$, but not on $x_{k}$.%

\begin{definition}
For $f:\mathcal{X}^{n}\rightarrow 
\mathbb{R}
$ we define the seminorms%
\begin{eqnarray*}
M\left( f\right) &=&\max_{k\in \left\{ 1,...,n\right\} }\sup_{x\in \mathcal{X%
}^{n},y,y^{\prime }\in \mathcal{X}}D_{y,y^{\prime }}^{k}f\left( x\right) \\
J\left( f\right) &=&n\max_{l,k:l\neq k}\sup_{x\in \mathcal{X}%
^{n},z,z^{\prime },y,y^{\prime }\in \mathcal{X}}D_{z,z^{\prime
}}^{l}D_{y,y^{\prime }}^{k}f\left( x\right) .
\end{eqnarray*}%
For $a,b>0$ we say that a function $f:\mathcal{X}^{n}\rightarrow 
\mathbb{R}
$ has $\left( a,b\right) $\textit{-weak interactions}, if $M\left( f\right)
\leq a/n$ and $J\left( f\right) \leq b/n$.

A sequence $\left( f_{n}\right) _{n\geq 2}$ of functions $f_{n}:\mathcal{X}%
^{n}\rightarrow 
\mathbb{R}
$ has $\left( a,b\right) $-weak interactions if every $f_{n}$ has $\left(
a,b\right) $-weak interactions.%
\end{definition}

The seminorm $M$ vanishes on constant functions, the seminorm $J$ vanishes
on additive functions. They can be interpreted as distribution-independent
distance measures to the linear subspaces of constant and additive functions
respectively. Notice the factor $n$ in the definition of $J$, so $%
D_{z,z^{\prime }}^{l}D_{y,y^{\prime }}^{k}f\left( x\right) \leq J\left(
f\right) /n$.

$M$ appears in the well known concentration inequality \citep{McDiarmid 1998,Boucheron13}
\begin{equation}
\Pr \left\{ f\left( \mathbf{X}\right) -E\left[ f\right] >t\right\} \leq \exp
\left( \frac{-2t^{2}}{nM\left( f\right) ^{2}}\right) ,
\label{McDiarmids inequality}
\end{equation}%
often called Bounded-Difference- or McDiarmid's inequality. This inequality
generalizes Hoeffding's inequality to general non-additive functions. Both
seminorms $M$ and $J$ appear in the recent inequality \citep{Maurer 2017}%
\begin{equation}
\Pr \left\{ f\left( \mathbf{X}\right) -E\left[ f\right] >t\right\} \leq \exp
\left( \frac{-2t^{2}}{2\sigma ^{2}\left( f\right) +J\left( f\right)
^{2}/2+\left( 2M\left( f\right) /3+J\left( f\right) \right) t}\right) ,
\label{Bernstein 1}
\end{equation}%
which generalizes Bernstein's inequality to non-additive functions.%

In this work we give an estimator $v_{f}$ for the variance $\sigma
^{2}\left( f\right) $, also in terms of $M$ and $J$ (Theorem \ref{Theorem
Variance Raw} below), which can be combined with inequality (\ref{Bernstein
1}) to a purely empirical bound, so as to improve over McDiarmid's
inequality for functions of small variance, just as the empirical Bernstein
bound for additive functions mentioned above. We also give a result for
normal approximation of general non-additive functions, also in terms of $M$
and $J$ (Theorem \ref{Theorem normal}), which can be converted to an
empirical result using our variance estimate.

If $M$ and $J$ cannot be appropriately controlled these results are useless.
But if a sequence of functions $\left( f_{n}\right) $ has weak interactions,
in the sense of above definition, then $M\left( f_{n}\right) $ and $J\left(
f_{n}\right) $ have linear or sublinear decay, and statistical properties
resemble that of the sample mean. This is intuitively understandable,
because $\left( f_{n}\right) $ approaches additivity (the mixed second
partial differences go to zero), $n$ times as fast as it becomes a constant
(the first partial differences go to zero). Section \ref{Section Properties
of Weak Interactions} contains our statistical results for general functions
and their specialization to functions with weak interactions.%

The class of functions with weak interactions contains U- and M-statistics
of any order and Lipschitz L-statistics. It also contains some more exotic
specimen, as error functionals for $\ell _{2}$-regularization or the
KL-divergence between the Gibbs-measures of true and empirical error for
Gibbs algorithms. Section \ref{Section Functions with weak interactions}
describes examples of weak interactions, all of which obey the results given
in Section \ref{Section Properties of Weak Interactions}. An appendix
contains proofs, other technical material, and a glossary of notation in
tabular form.

\section{Bounds for functions with weak interactions\label{Section Properties of Weak
Interactions}}

In this section, we give some statistical properties of the random variable $%
f\left( \mathbf{X}\right) $ and specialize them to functions with weak
interactions $\left( a,b\right) $, so as to make them directly applicable to
the examples in Section \ref{Section Functions with weak interactions}.

\subsection{Notation, the Efron-Stein and Bernstein inequalities \label%
{Subsection EfronStein and Bernstein}}

In the sequel, $\mathcal{X}$ will be a measurable space and $\left( \mu
_{k}\right) _{k\geq 1}$ a sequence of probability measures on $\mathcal{X}$.
The random variables distributed as $\mu _{k}$ are independent and denoted $%
X_{k}$ or $X_{k}\sim \mu _{k}$ or $\left( X_{1},...,X_{n}\right) \sim
\prod_{1}^{n}\mu _{k}$. They are not identically distributed ($\mu _{k}=\mu $%
) unless explicitely mentioned. With $\mathbf{x}$ we denote a vector of the
form $\left( x_{1},...,x_{n}\right) \in \mathcal{X}^{n}$ and with $\mathbf{X}
$ a random vector of the form $\left( X_{1},...,X_{n}\right) \sim
\prod_{k=1}^{n}\mu _{k}$. The algebra of bounded measurable functions $g:%
\mathcal{X}^{n}\rightarrow 
\mathbb{R}
$ will be denoted by $\mathcal{A}_{n}$. If $g\in \mathcal{A}_{n}$ and if $%
\mathbf{x}$ has at least $n$ components, then $g\left( \mathbf{x}\right) $
is the function value $g\left( x_{1},...,x_{n}\right) $, and if $\mathbf{X}$
has at least $n$ components then $g\left( \mathbf{X}\right) $ is the random
variable $g\left( X_{1},...,X_{n}\right) $. For $g\in \mathcal{A}_{n}$
expectation and variance of $g\left( \mathbf{X}\right) $ will be abbreviated
by $E\left[ g\right] $ and $\sigma ^{2}\left( g\right) $. A function $g\in 
\mathcal{A}_{n}$ is called additive if $f\left( \mathbf{x}\right)
=\sum_{i=1}^{n}h_{i}\left( x_{i}\right) $ for some real valued $h_{i}:\mathcal{%
X\rightarrow 
\mathbb{R}
}$.

For $f\in $ $\mathcal{A}_{n}$ the $k$-th conditional variance $\sigma
_{k}^{2}\left( f\right) $ and the sum of conditional variances $\Sigma
^{2}\left( f\right) $ are the members of $\mathcal{A}_{n}$ defined by%
\begin{eqnarray*}
\sigma _{k}^{2}\left( f\right) \left( \mathbf{x}\right) &=&\frac{1}{2}%
E_{\left( Y,Y^{\prime }\right) \sim \mu _{k}\times \mu _{k}}\left[ \left(
D_{Y,Y^{\prime }}^{k}f\left( \mathbf{x}\right) \right) ^{2}\right] \\
\Sigma ^{2}\left( f\right) \left( \mathbf{x}\right) &=&\sum_{k=1}^{n}\sigma
_{k}^{2}\left( f\right) \left( \mathbf{x}\right) .
\end{eqnarray*}%
Note that $\sigma _{k}^{2}\left( f\right) $ does not depend on $x_{k}$, that 
$\sigma _{k}^{2}\left( f\right) \left( \mathbf{x}\right) \leq M\left(
f\right) ^{2}/4$ (because the variance of a bounded random variable is
always bounded by a quarter of the square of its range) and that $\Sigma
^{2}\left( f\right) \left( \mathbf{x}\right) \leq nM\left( f\right) ^{2}/4$.
For additive functions $\Sigma ^{2}\left( f\right) \left( \mathbf{x}\right) $
is independent of $\mathbf{x}$ and equals $\sigma ^{2}\left( f\right) $. For
non-additive functions this does not hold any more, instead one has the
Efron-Stein inequality \citep{Efron 1981,Steele 1986} 
\begin{equation}
\sigma ^{2}\left( f\right) \leq E\left[ \Sigma ^{2}\left( f\right) \right] ,
\label{Efron-Stein Inequality}
\end{equation}%
which gives the general bound $\sigma ^{2}\left( f\right) \leq nM\left(
f\right) ^{2}/4$ on the variance. For functions with $M\left( f\right) \leq
a/n$ (in particular for weak interactions) we get%
\begin{equation}
\sigma ^{2}\left( f\right) \leq \frac{a^{2}}{4n}.
\label{A priori variance bound}
\end{equation}%
The Efron-Stein inequality is very sharp for functions with weak
interactions. We have%
\begin{eqnarray}
E\left[ \Sigma ^{2}\left( f\right) \right] &\leq &\sigma ^{2}\left( f\right)
+\frac{1}{4}\sum_{k,l:k\neq l}E_{\mathbf{X},Z,Z^{\prime },Y,Y^{\prime }}%
\left[ \left( D_{ZZ^{\prime }}^{l}D_{YY^{\prime }}^{k}f\left( \mathbf{X}%
\right) \right) ^{2}\right]  \notag \\
&\leq &\sigma ^{2}\left( f\right) +\frac{J\left( f\right) ^{2}}{4}.
\label{Houdre bound general}
\end{eqnarray}%
The first inequality is due to \citet{Houdre 1997} \cite[see also][]{Maurer
2017}, the second is an elementary estimate. For weak interactions we get 
\begin{equation}
E\left[ \Sigma ^{2}\left( f\right) \right] -\frac{b^{2}}{4n^{2}}\leq \sigma
^{2}\left( f\right) \leq E\left[ \Sigma ^{2}\left( f\right) \right] .
\label{Houdre Bound}
\end{equation}%
In \cite{Maurer 2017} the following Bernstein-type inequality is shown to
hold for every $f$ in $\mathcal{A}_{n}$ and $\delta >0$%
\begin{equation}
\Pr \left\{ f-E\left[ f\right] >\sqrt{2E\left[ \Sigma ^{2}\left( f\right) %
\right] \ln \left( 1/\delta \right) }+\left( 2M\left( f\right) /3+J\left(
f\right) \right) \ln \left( 1/\delta \right) \right\} <\delta .
\label{Bernstein inequality}
\end{equation}%
Using (\ref{Houdre Bound}) and some elementary estimates for functions with $%
\left( a,b\right) $-weak interactions we obtain for $\delta \leq 1/e$%
\begin{equation*}
\Pr \left\{ f-E\left[ f\right] >\sqrt{2\sigma ^{2}\left( f\right) \ln \left(
1/\delta \right) }+\left( 2a/3+3b/2\right) \frac{\ln \left( 1/\delta \right) 
}{n}\right\} <\delta . 
\end{equation*}%
Since $\sigma \left( f\right) $ decays at least as quickly as $a/\sqrt{4n}$
because of (\ref{A priori variance bound}), this achieves, for large $n$, at
least the rate of McDiarmid's inequality (\ref{McDiarmids inequality}), but
it is potentially much better if $\sigma \left( f\right) $ is very small.
This motivates the search for efficient estimators of $\sigma \left(
f\right) $.

\subsection{Variance estimation\label{Subsection VarianceEstimation}}

We show that for $f\in \mathcal{A}_{n}$ having $\left( a,b\right) $-weak
interactions $\sigma \left( f\right) $ can be estimated with high
probability up to order $1/n$ by an estimator using only $n+1$ observations.
This is one of the main results of this work.

For any $n>1$, $1\leq k\leq n$, $\mathbf{x}\in \mathcal{X}^{n}$ and $y\in 
\mathcal{X}$ define the replacement operator $S_{y}^{k}:\mathcal{X}%
^{n}\rightarrow \mathcal{X}^{n}$ and the deletion operator $S_{-}^{k}:%
\mathcal{X}^{n}\rightarrow \mathcal{X}^{n-1}$ by%
\begin{eqnarray*}
S_{y}^{k}\mathbf{x} &=&\left( x_{1},...,x_{k-1},y,x_{k+1},...,x_{n}\right)
\in \mathcal{X}^{n} 
\end{eqnarray*}%
and
\begin{eqnarray*}
S_{-}^{k}\mathbf{x} &=&\left(
x_{1},...,x_{k-1},x_{k+1},...,x_{n}\right) \in \mathcal{X}^{n-1}.
\end{eqnarray*}%
Our variance estimator is the function $v_{f}\in \mathcal{A}_{n+1}$ given by%
\begin{equation}
v_{f}\left( \mathbf{x}\right) =\frac{1}{2\left( n+1\right) }%
\sum_{j=1}^{n+1}\sum_{i:i\neq j}\left( f\left( S_{-}^{j}\mathbf{x}\right)
-f\left( S_{-}^{j}S_{x_{j}}^{i}\mathbf{x}\right) \right) ^{2}.
\label{Variance estimator}
\end{equation}

$S_{-}^{j}\mathbf{x}$ has the $j$-th component deleted and $%
S_{-}^{j}S_{x_{j}}^{i}\mathbf{x}$ has the $i$-th component replaced by the $%
j $-th component and then the $j$-th component deleted. So both vectors
differ only in one component, which is $x_{i}$ in $S_{-}^{j}\mathbf{x}$ and $%
x_{j}$ in $S_{-}^{j}S_{x_{j}}^{i}\mathbf{x}$. Also both vectors do not
contain any repeated components.

It is obvious how the estimator is to be implemented in a computer program.
Computation requires $\left( n+1\right) ^{2}$ computations of $f$, but only
a sample of size $n+1$. The latter may be a great advantage, because
computing may be cheap, while collecting a sample can be very expensive
(think of surveys or the results of histological examinations in medical
applications). We first give the result in terms of the seminorms.

\begin{theorem}
\label{Theorem Variance Raw}Let $\delta \in (0,1)$. If $f\in \mathcal{A}_{n}$ and the $X_{i}$ are
identically distributed, then with probability at least $%
1-\delta $ in $\mathbf{X}=\left( X_{1},...,X_{n+1}\right) $%
\begin{equation*}
\left\vert \sqrt{E\left[ \Sigma ^{2}\left( f\right) \right] }-\sqrt{%
v_{f}\left( \mathbf{X}\right) }\right\vert \leq \sqrt{\left( 2M\left(
f\right) ^{2}+8J\left( f\right) ^{2}\right) \ln \left( 2/\delta \right) }. 
\end{equation*}%
For one-sided bounds $2/\delta $ can be replaced by $1/\delta $.
\end{theorem}
The proof is given in the appendix, Section \ref{Subsection
Proof Variance estimation}. It first establishes that $v_{f}$ is an unbiased
estimator for $E\left[ \Sigma ^{2}\left( f\right) \right] $ and then uses a
concentration inequality for self-bounded functions.

The result requires identical distribution of the $X_{i}$, in contrast to
the Efron-Stein and Bernstein inequalities, but it does not require $f$ to
be symmetric.  It is important to observe that our estimator requires one additional
observation, as the the variance of $f\left( X_{1},...,X_{n}\right) $ is estimated by $v_{f}\left( X_{1},...,X_{n},X_{n+1}\right)$.

Because of (\ref{Houdre bound general}) and (\ref{Efron-Stein Inequality})
we have $E\left[ \Sigma ^{2}\left( f\right) \right] -J\left( f\right)
^{2}/4\leq \sigma ^{2}\left( f\right) \leq E\left[ \Sigma ^{2}\left(
f\right) \right] $. Thus%
\begin{equation*}
\Pr \left\{ \sqrt{v_{f}\left( \mathbf{X}\right) }-J\left( f\right) /2-\sqrt{%
\left( 2M\left( f\right) ^{2}+8J\left( f\right) ^{2}\right) \ln \left(
2/\delta \right) }<\sigma \left( f\right) \right\} <\delta /2, 
\end{equation*}%
which together with Theorem \ref{Theorem Variance Raw} immediately gives the
following corollary (using $\delta <1\implies 1/2\leq \sqrt{\ln \left(
2/\delta \right) }$).

\begin{corollary}
\label{Corollary weak interaction variance}Let $\delta \in (0,1)$. If $f\in \mathcal{A}_{n}$ has $%
\left( a,b\right) $-weak interactions and the $X_{i}$ are identically
distributed, then with probability at least $1-\delta $ in $%
\mathbf{X}=\left( X_{1},...,X_{n+1}\right) $%
\begin{equation*}
\sqrt{v_{f}\left( \mathbf{X}\right) }-\frac{K_{-}\left( a,b\right) }{n}\sqrt{%
\ln \left( 2/\delta \right) }\leq \sigma \left( f\right) \leq \sqrt{%
v_{f}\left( \mathbf{X}\right) }+\frac{K_{+}\left( a,b\right) }{n}\sqrt{\ln
\left( 2/\delta \right) }, 
\end{equation*}%
where%
\begin{eqnarray*}
K_{-}\left( a,b\right) & = & b/2+\sqrt{2a^{2}+8b^{2}}\\
K_{+}\left(a,b\right) & = & \sqrt{2a^{2}+8b^{2}}. 
\end{eqnarray*}
For one-sided bounds $2/\delta $ can be replaced by $1/\delta $.
\end{corollary}
The bounds on the variance are of order $1/n$. Since the Efron-Stein
inequality implies only $\sigma \left( f\right) \leq a/\sqrt{4n}$, there is
a significant estimation benefit for larger values of $n$.

If $f$ is the sample mean (\ref{Sample mean}), then 
\begin{equation*}
f\left( S_{-}^{j}\mathbf{x}\right) -f\left( S_{-}^{j}S_{x_{j}}^{i}\mathbf{x}%
\right) =\frac{1}{n}\left\{ 
\begin{array}{ccc}
x_{i}-x_{j} & \text{if} & i<j \\ 
x_{i-1}-x_{j-1} & \text{if} & j<i%
\end{array}%
\right. , 
\end{equation*}%
so substitution in (\ref{Variance estimator}) shows that the estimator $%
v_{f}=\left( 1/n\right) v_{n+1}$, where $v_{n+1}$ is the sample variance (%
\ref{Sample variance}). Since $b=0$ for the sample mean we get the bound%
\begin{equation*}
\left\vert \sqrt{v_{f}\left( \mathbf{X}\right) }-\sigma \left( f\right)
\right\vert \leq \frac{1}{n}\sqrt{2\ln \left( 2/\delta \right) }, 
\end{equation*}%
so for the sample mean Corollary \ref{Corollary weak interaction variance}
gives the same rate as \citep{Maurer 2009}.

\subsection{Normal approximation\label{Subsection Normal Approximation}}

Modulo a lower bound on the variance, we give a finite sample bound on
normal approximation for functions with weak interactions. To formulate the
result we use the following distance to normality of a real random variable $%
W.$%
\begin{equation*}
d_{\mathcal{N}}\left( W\right) =\sup \left\{ \left\vert E\left[ h\left( 
\frac{W-E\left[ W\right] }{\sigma \left( W\right) }\right) \right] -E\left[
h\left( Z\right) \right] \right\vert :h\text{ a real Lipschitz-1 function}%
\right\} , 
\end{equation*}%
where $Z\sim \mathcal{N}\left( 0,1\right) $. Thus $d_{\mathcal{N}}\left(
W\right) $, which has also been used in \cite{Chatterjee 2008}, is the
Wasserstein distance between a standardized clone of $W$ and a standard
normal variable. We then have the following general result.

\begin{theorem}
\label{Theorem normal}For $f\in \mathcal{A}_{n}$ 
\begin{equation*}
d_{\mathcal{N}}\left( f\left( \mathbf{X}\right) \right) \leq \frac{\sqrt{n}%
M\left( f\right) \left( J\left( f\right) +M\left( f\right) \right) }{\sigma
^{2}\left( f\right) }+\frac{nM\left( f\right) ^{3}}{2\sigma ^{3}\left(
f\right) }. 
\end{equation*}
\end{theorem}

The proof is given in the appendix, Section \ref{Subsection Proof normal
approximation}. It relies on an inequality of \citet{Chatterjee
2008}, which uses a variant of Stein's method \citep{Chen 2010} for normal
approximation. To apply the result we need a lower bound on the variance. In
the next section we use an empirical estimate, but here we simply assume a
bound of the form $\sigma \left( f\right) \geq Cn^{-p}$ for some constants $%
C $ and $p$. By (\ref{A priori variance bound}) we must have $p\geq 1/2$.
Specializing to weak interactions we obtain with some algebra

\begin{corollary}
\label{Corollary weak interaction normal approximation}If $f\in \mathcal{A}%
_{n}$ has $\left( a,b\right) $-weak interactions and $\sigma \left( f\right) \geq Cn^{-p}$
then%
\begin{equation*}
d_{\mathcal{N}}\left( f\left( \mathbf{X}\right) \right) \leq \frac{Ca\left(
a+b\right) +a^{3}}{C^{3}n^{2-3p}}. 
\end{equation*}
\end{corollary}

So if a sequence $f_{n}$ has $\left( a,b\right) $-weak interactions, $\sigma
\left( f_{n}\right) \geq Cn^{-p}$ and $1/2\leq p<2/3$, then the sequence $%
\left( f_{n}\left( \mathbf{X}\right) -E\left[ f_{n}\right] \right) /\sigma
\left( f_{n}\right) $ converges to a standard normal variable in the
Wasserstein metric. For $p\geq 2/3$ the result says nothing about the
asymptotic distribution. In the simplest case $p=1/2$ the rate of approach
to normality is $n^{-1/2}$.

\subsection{Empirical bounds for weak interactions\label{Subsection
empirical Bounds}}

Now we will cast the Bernstein inequality (\ref{Bernstein inequality}) and
the normal approximation inequality of the previous section into an
empirical form by using the results on variance estimation of Section \ref%
{Subsection VarianceEstimation}. In this case we will need identical
distribution of the variables $X_{i}$.

To combine Bernstein's inequality (\ref{Bernstein inequality}) and the upper
bound on the variance of Corollary \ref{Corollary weak interaction variance}
elementary estimates give

\begin{theorem}
\label{Theorem empirical Bernstein Bound}{\em (Empirical Bernstein Inequality)}
If $f\in \mathcal{A}_{n}$ has $\left( a,b\right) $-weak interactions and the 
$X_{i}$ are iid, then for $\delta >0$ with probability at least $1-\delta $%
\begin{equation*}
f\left( \mathbf{X}\right) \leq E\left[ f\right] +\sqrt{2v_{f}\left( \mathbf{X%
}\right) \ln \left( 2/\delta \right) }+\frac{\left( 8a/3+5b\right) \ln
\left( 2/\delta \right) }{n}. 
\end{equation*}
\end{theorem}

While for Bernstein's inequality we want the variance to be small, for our
normal approximation result, Theorem \ref{Theorem normal}, we want it to be
big. The situation is also more complicated, because the variance now
appears in the denominator of the bound, so the estimate may fail. In fact
it may even fail for all members of a sequence, because asymptotic normality
needn't hold. We therefore precede the empirical bound by a test to verify
its applicability.

\begin{theorem}
\label{Theorem empirical normal approximation}Suppose that $f\in \mathcal{A}%
_{n}$ has $\left( a,b\right) $-weak interactions and the $X_{i}$ are iid.
For $\delta >0$ let $A\left( \delta \right) $ and $B$ be the events 
\begin{eqnarray*}
A\left( \delta \right) &=&\left\{ \frac{\sqrt{v_{f}\left( \mathbf{X}\right) }%
}{2}\geq \frac{K_{-}\left( a,b\right) \sqrt{\ln \left( 1/\delta \right) }}{n}%
\right\} , \\
B &=&\left\{ d_{\mathcal{N}}\left( f\left( \mathbf{X}\right) \right) \leq 
\frac{4\left( a^{2}+ab\right) }{v_{f}\left( \mathbf{X}\right) n^{3/2}}+\frac{%
4a^{3}}{v_{f}\left( \mathbf{X}\right) ^{3/2}n^{2}}\right\} .
\end{eqnarray*}%
Then $\Pr \left( A\left( \delta \right) \implies B\right) \geq 1-\delta $.
\end{theorem}

\begin{proof}
Let $C\left( \delta \right) $ be the event%
\begin{equation*}
C\left( \delta \right) =\left\{ \sqrt{v_{f}\left( \mathbf{X}\right) }-\frac{%
K_{-}\left( a,b\right) \sqrt{\ln \left( 1/\delta \right) }}{n}\leq \sigma
\left( f\right) \right\} . 
\end{equation*}%
Then by Corollary \ref{Corollary weak interaction variance} $\Pr C\left(
\delta \right) \geq 1-\delta $. But under $C\left( \delta \right) $ the
event $A$ implies%
\begin{equation*}
\frac{\sqrt{v_{f}\left( \mathbf{X}\right) }}{2}\leq \sqrt{v_{f}\left( 
\mathbf{X}\right) }-\frac{K_{-}\left( a,b\right) \sqrt{\ln \left( 1/\delta
\right) }}{n}\leq \sigma \left( f\right) 
\end{equation*}%
which implies $B$ by Theorem \ref{Theorem normal} and $\left( a,b\right) $%
-weak interactions of $f$.
\end{proof}

On a sequence of functions $f_{n}$ this result could be put to work as
follows. First fix $\delta $ and $n$ and observe $\mathbf{X}_{1}^{n+1}$.
Then compute the variance estimator and check if $A\left( \delta \right) $
holds. If it doesn't hold then $n$ may be to small and we may try a larger $%
n $. If we don't get it too work then the variances decay too fast and $f_{n}
$ may simply not be asymptotically normal, so we give up. If $A\left( \delta
\right) $ holds on the other hand, we have an empirical bound on normal
approximation, which can tell us a lot about the distribution of $f\left( 
\mathbf{X}\right) $.

In the regime where Corollary \ref{Corollary weak interaction normal
approximation} guarantees asymptotic normality, that is $\sigma \left(
f_{n}\right) \geq Cn^{-p}$ and $1/2\leq p<2/3$, Corollary \ref{Corollary
weak interaction variance} guarantees that the test $A\left( \delta \right) $
succeeds with high probability for sufficiently large $n$.

\section{Examples of functions with weak interactions\label{Section Functions with weak
interactions}}

We give examples of functions having weak interactions and identify the
parameters $\left( a,b\right) $, so as to make the results of the previous
section applicable. Some obvious closure relations for functions with weak
interactions follow from the fact that $M$ and $J$ are seminorms. If $f_{1}$
and $f_{2}$ have $\left( a_{1},b_{1}\right) $- and $\left(
a_{2},b_{2}\right) $-weak interactions respectively and $c\in 
\mathbb{R}
$, then $f_{1}+f_{2}$ has $\left( a_{1}+a_{2},b_{1}+b_{2}\right) $-weak
interactions, $f_{1}+c$ has $\left( a_{1},b_{1}\right) $-weak interactions
and $cf_{1}$ has $\left( \vert c\vert a,\vert c\vert b\right) $-weak interactions. The last fact
allows to rescale the conveniently scaled examples we choose below.

\subsection{The sample mean, V- and U-statistics\label{Subsection MandU
statistics}}

Let $\mathcal{X=}\left[ 0,1\right] $. The sample mean%
\begin{equation*}
f\left( \mathbf{x}\right) =\frac{1}{n}\sum_{i=1}^{n}x_{i} 
\end{equation*}%
has seminorm values $M\left( f\right) =1/n$ and $J\left( f\right) =0$, and
therefore $\left( 1,0\right) $-weak interactions. $f\left( \mathbf{X}\right) 
$ is an unbiased estimator of the expectation of a $\left[ 0,1\right] $%
-valued random variable.

V- and U-statistics are generalizations of the sample mean. Fix $1\leq m<n$,
and for any multi-index $\mathbf{j}=\left( j_{1},...,j_{m}\right) \in
\left\{ 1,...,n\right\} ^{m}$ let $\kappa _{\mathbf{j}}:\mathcal{X}%
^{m}\rightarrow \left[ -1,1\right] $ and define $V,U:\mathcal{X}%
^{m}\rightarrow 
\mathbb{R}
$,%
\begin{eqnarray*}
V\left( \mathbf{x}\right) &=&n^{-m}\sum_{\mathbf{j}\in \left\{
1,...,n\right\} ^{m}}\kappa _{\mathbf{j}}\left(
x_{j_{1}},...,x_{j_{m}}\right) \\
U\left( \mathbf{x}\right) &=&\binom{n}{m}^{-1}\sum_{1\leq
j_{1}<...<j_{m}\leq n}\kappa _{\mathbf{j}}\left(
x_{j_{1}},...,x_{j_{m}}\right) .
\end{eqnarray*}%
V-statistics have their name from Richard von Mises, who studied their
asymptotic distributions \citep{Van Mises}. $V\left( \mathbf{x}\right) $
receives contributions from multi-indices with multiple occurrences of
individual indices. But in the expression for $D_{y,y^{\prime }}^{k}V\left( 
\mathbf{x}\right) $ only those multi-indices $\mathbf{j}$ survive, which
contain $k$, with the corresponding contribution being at most $2n^{-m}$.
There is a first position where $k$ appears in $\mathbf{j}$, for which there
are $m$ possibilities, and the remaining indices $j_{i}$ can assume all
values in $\left\{ 1,...,n\right\} $. It follows that there are at most $%
mn^{m-1}$ surviving multi-indices with maximal contribution $2n^{-m}$, whence%
\begin{equation*}
M\left( V\right) =\max_{k}\sup_{\mathbf{x},y,y^{\prime }}D_{y,y^{\prime
}}^{k}V\left( \mathbf{x}\right) \leq \frac{2mn^{m-1}}{n^{m}}=\frac{2m}{n}. 
\end{equation*}%
For $D_{z,z^{\prime }}^{l}D_{y,y^{\prime }}^{k}V\left( \mathbf{x}\right) $
with $k\neq l$ each contributing index must contain both $k$ and $l$. For
the positions of $k$ and $l$ there are $m\left( m-1\right) $ possibilities.
The remaining $m-2$ indices being arbitrary, there is a total of at most $%
m\left( m-1\right) n^{m-2}$ contributing indices, each making a contribution
of at most $4n^{-m}$. Therefore $D_{z,z^{\prime }}^{l}D_{y,y^{\prime
}}^{k}V\left( \mathbf{x}\right) \leq 4m\left( m-1\right) /n^{2}$ and%
\begin{equation*}
J\left( V\right) =n~\max_{k\neq l}\sup_{\mathbf{x},z,z^{\prime },y,y^{\prime
}}D_{z,z^{\prime }}^{l}D_{y,y^{\prime }}^{k}V\left( \mathbf{x}\right) \leq
4m\left( m-1\right) /n. 
\end{equation*}%
We conclude that $V$ has $\left( 2m,4m\left( m-1\right) \right) $-weak
interactions.

U-statistics avoid multi-indices with multiple occurrences of indices. If
all the $\kappa _{\mathbf{j}}$ are equal to some permutation symmetric
function $\kappa $, and the $X_{i}$ are iid, then $U\left( \mathbf{X}\right) 
$ is an unbiased estimator for $E\left( X_{1},...,X_{m}\right) $, which
accounts for their name \citep{Hoeffding 1948}. U-statistics are relevant to
metric learning \citep{Cao 2016} and ranking \citep{Clemencon 2008}. Similar
to $V$-statistics it is not difficult to show that $U$ has $\left(
2m,4m^{2}\right) $-weak interactions \citep[see][]{Maurer 2017}.

U-statistics have been extensively studied. There are normal approximation
results for nondegenerate U-statistics in \cite{Chen 2010}, which
use the Kolmogorov distance and seem to slightly improve over what we get
from substituting $\left( 2m,4m^{2}\right) $ in Corollary \ref{Corollary
weak interaction normal approximation}. These results also contain
variances, which would make them amenable to variance estimation as in
Theorem \ref{Theorem empirical normal approximation}.

\citet{Peel 2010} use the fact that the variance of a U-statistic
is itself a U-statistic and use either \citet{Hoeffding 1948} or
\citet{Arcones 1995} versions of Bernstein's inequality for
U-statistics to estimate the variance. These bounds are however inferior to
the Bernstein inequality (\ref{Bernstein inequality}), because the first
does not use the correct variance proxy and the second has a scale proxy
which increases exponentially in the order $m$. The same problem besets the
empirical Bernstein bounds given in \citep{Peel 2010}, which is inferior to
the general result we get from Theorem \ref{Theorem empirical Bernstein
Bound} except for the first version of \citet{Peel 2010} in a regime of large 
$m/n$ and a kernel $\kappa $ far from degeneracy.

\subsection{Lipschitz L-statistics\label{Subsection Lipschitz L-statistics}}

Let $\mathcal{X}=\left[ 0,1\right] $ and use $\left( x_{\left( 1\right)
},...,x_{\left( n\right) }\right) $ to denote the order statistic of $%
\mathbf{x\in }\mathcal{X}^{n}$. Let $F:\left[ 0,1\right] \rightarrow 
\mathbb{R}
$ have supremum norm $\left\Vert F\right\Vert _{\infty }$ and
Lipschitz-constant $\left\Vert F\right\Vert _{Lip}$ and consider the function%
\begin{equation}
f\left( \mathbf{x}\right) =\frac{1}{n}\sum_{i=1}^{n}F\left( i/n\right)
x_{\left( i\right) }\text{.}  \label{L-statistic}
\end{equation}%
In the appendix, Section \ref{Subsection Argument L-statistics} we show that $f$ has $\left( \left\Vert F\right\Vert _{\infty },\left\Vert
F\right\Vert _{Lip}\right) $-weak interactions. \newline
Such statistics also generalize the sample mean, which is obtained by
choosing $F$ identically $1$. Appropriate choices of $F$ lead to smoothly
trimmed means or smoothened quantiles. For example with $\zeta \in \left(
0,1/2\right) $ the choice 
\begin{equation*}
F\left( t\right) =\left\{ 
\begin{array}{ccc}
0 & \text{if} & t\in \left[ 0,\zeta \right] \\ 
\frac{t-\zeta }{\left( 1/2-\zeta \right) ^{2}} & \text{if} & t\in \left[
\zeta ,1/2\right] \\ 
\frac{1-t-\zeta }{\left( 1/2-\zeta \right) ^{2}} & \text{if} & t\in \left[
1/2,1-\zeta \right] \\ 
0 & \text{if} & t\in \left[ 1-\zeta \right]%
\end{array}%
\right. 
\end{equation*}%
effects a smoothened median where $F$ has the Lipschitz constant $\left\Vert
F\right\Vert _{Lip}=\left( 1/2-\zeta \right) ^{-2}$. The case $\zeta =0$ has the best guaranteed
estimation properties, but its expectation is the coarsest substitute of the
median. As $\zeta \rightarrow 1/2$ estimation deteriorates, but the expectation becomes closer to a median.

Normal approximation results for these statistics in terms of the Kolmogorov
distance are also given in \cite{Chen 2010}, similar to what we
obtain by substituting $\left( \left\Vert F\right\Vert _{\infty },\left\Vert
F\right\Vert _{Lip}\right) $ in Corollary \ref{Corollary weak
interaction normal approximation}. We are not aware of any results giving
Bernstein-type inequalities or tight variance estimation in this case.

\subsection{$\ell_2$-regularization\label{Subsection L_2 regularization}}

While the previous examples had a certain kinship to the sample mean, the
following looks quite different. Let $\left( H,\left\langle \cdot,\cdot\right\rangle
,\left\Vert \cdot\right\Vert \right) $ be a real Hilbert space with unit ball $%
\mathbb{B}_{1}=\mathcal{X}$ and define $g:$ $\mathcal{X}^{n}\rightarrow H$ by%
\begin{equation}
g\left( \mathbf{x}\right) =\arg \min_{w\in H}\frac{1}{n}\sum_{i=1}^{n}\ell
\left( \left\langle x_{i},w\right\rangle \right) +\lambda \left\Vert
w\right\Vert ^{2},  \label{Define g}
\end{equation}%
where the non-negative real loss function $\ell $ is assumed to be convex,
three times differentiable and satisfies $\ell \left( 0\right) =1$, and the
regularization parameter $\lambda $ satisfies $0<\lambda <1$. Then $g$ is a
well-known regularized algorithm which upon thresholding can be used for
linear classification.

Define the empirical and the true losses $\hat{L}$ and $L:\mathcal{X}%
^{n}\rightarrow 
\mathbb{R}
$ by%
\begin{equation*}
\hat{L}\left( \mathbf{x}\right) =\frac{1}{n}\sum_{i}\ell \left( \left\langle
x_{i},g\left( \mathbf{x}\right) \right\rangle \right) \text{ and }L\left( 
\mathbf{x}\right) =E_{x\sim \mu }\left[ \ell \left( \left\langle x,g\left( 
\mathbf{x}\right) \right\rangle \right) \right] , 
\end{equation*}%
where $\mu $ is some probability measure on $\mathbb{B}_{1}$. Let%
\begin{equation*}
\Delta \left( \mathbf{x}\right) =L\left( \mathbf{x}\right) -\hat{L}\left( 
\mathbf{x}\right) , 
\end{equation*}%
which measures how much the true and empirical loss of the algorithm differ.
It has been shown in \cite[Proposition 5][]{Maurer2017b}, that $\Delta $ has $%
\left( c_{1}\lambda ^{-3/2},c_{2}\lambda ^{-4}\right) $-weak interactions,
where the constants $c_{i}$ depend on the derivatives of the loss-function $%
\ell $. In \citep{Maurer2017b} this is used to apply the Bernstein inequality
(\ref{Bernstein inequality}) to the random variable $\Delta \left( \mathbf{X}%
\right) $. Here we complement this result by simply substituting the weak
interaction parameters in Corollary \ref{Corollary weak interaction variance}
and Corollary \ref{Corollary weak interaction normal approximation} so as to
obtain bounds to estimate the variance of $\Delta \left( \mathbf{X}\right) $
and to give bounds on normal approximation.

\subsection{A chain rule\label{Subsection Chain rule}}

We interrupt the presentation of examples, to show how new interesting
examples of functions with weak interactions can be generated from given
ones, in addition to the obvious closure relations which follow from $M$ and 
$J$ being seminorms. First we extend the definitions of $M$ and $J$ to
Banach-space valued functions $f:\mathcal{X}^{n}\rightarrow B$ in an obvious
way by setting%
\begin{equation*}
M\left( f\right) =\max_{k}\sup_{\mathbf{x},y,y^{\prime }}\left\Vert
D_{yy^{\prime }}^{k}f\left( x\right) \right\Vert \text{ and }J\left(
f\right) =n~\max_{k\neq l}\sup_{\mathbf{x},y,y^{\prime },z,z^{\prime
}}\left\Vert D_{zz^{\prime }}^{l}D_{yy^{\prime }}^{k}f\left( x\right)
\right\Vert ,
\end{equation*}%
and we say that $f$ has $\left( a,b\right) $-weak interactions if $M\left(
f\right) \leq a/n$ and $J\left( f\right) \leq b/n$. Then we have the
following chain rule, whose proof will be given in the appendix, Section \ref%
{Subsection Proof Chainrule}.

\begin{lemma}
\label{Lemma chain rule}Let $B$ be a Banach space, $U\subseteq B$ convex, $f:%
\mathcal{X}^{n}\rightarrow U$, and assume that the function $F:U\rightarrow 
\mathbb{R}
$ is twice Fr\'{e}chet-differentiable. Then%
\begin{eqnarray*}
M\left( F\circ f\right) &\leq &\sup_{v\in U}\left\Vert F^{\prime }\left(
v\right) \right\Vert M\left( f\right) \text{ and} \\
J\left( F\circ f\right) &\leq &n\sup_{v\in U}\left\Vert F^{\prime \prime
}\left( v\right) \right\Vert M\left( f\right) ^{2}+\sup_{v\in U}\left\Vert
F^{\prime }\left( v\right) \right\Vert J\left( f\right) ,
\end{eqnarray*}%
where $\left\Vert F^{\prime }\left( v\right) \right\Vert $ and $\left\Vert
F^{\prime \prime }\left( v\right) \right\Vert $ are the norms of the linear
respectively bilinear functionals $F^{\prime }\left( v\right) $ and $%
F^{\prime \prime }\left( v\right) $.%
\begin{equation*}
\left\Vert F^{\prime }\left( v\right) \right\Vert =\sup_{w\in B,\left\Vert
w\right\Vert \leq 1}\left\Vert F^{\prime }\left( v\right) \left( w\right)
\right\Vert \text{ and }\left\Vert F^{\prime \prime }\left( v\right)
\right\Vert =\sup_{w_{1},w_{2}\in B,\left\Vert w_{i}\right\Vert \leq
1}\left\Vert F^{\prime \prime }\left( v\right) \left( w_{1},w_{2}\right)
\right\Vert \text{.} 
\end{equation*}
\end{lemma}

The lemma shows that if $f$ has $\left( a,b\right) $-weak interactions and $%
\left\Vert F^{\prime \prime }\left( v\right) \right\Vert $ and $\left\Vert
F^{\prime }\left( v\right) \right\Vert $ are bounded on $U$, then $F\circ f$
has $\left( a^{\prime },b^{\prime }\right) $-weak interactions, where%
\begin{equation*}
a^{\prime }=a\sup_{v\in U}\left\Vert F^{\prime }\left( v\right) \right\Vert 
\text{ and }b^{\prime }=a^{2}\sup_{v\in U}\left\Vert F^{\prime \prime
}\left( v\right) \right\Vert +b\sup_{v\in U}\left\Vert F^{\prime }\left(
v\right) \right\Vert . 
\end{equation*}
It also shows our definition of weak interactions with its $1/n$-scaling is
the only definition of a class of functions such that $M$ and $J$ are of the same order in $n$, and the class is invariant under compositions with $C^{2}$ functions with bounded derivatives.

\subsection{The Gibbs algorithm\label{Subsection Gibbs algorithm}}

We use the chain rule, Lemma \ref{Lemma chain rule}, to show that several quantities
related to the Gibbs algorithm have weak interactions and thus satisfy the
conditions for the results in Section \ref{Section Properties of Weak
Interactions}.

Let $\Omega $ be some space of ``models" endowed with some positive a-priori
measure $\rho $ and suppose that $\ell :\left( \omega ,x\right) \in \Omega
\times \mathcal{X}\mapsto \ell \left( \omega ,x\right) \in \left[ 0,1\right] 
$ is the loss of the model $\omega $ on the datum $x\in \mathcal{X}$. The
function $H:\Omega \times \mathcal{X}^{n}\rightarrow \left[ 0,1\right] $
defined by $H\left( \omega ,\mathbf{x}\right) =\left( 1/n\right)
\sum_{i=1}^n\ell \left( \omega ,x_{i}\right) $ is then just the sample average,
or empirical error of $\omega $ on $\mathbf{x}$. Let $\beta $ be some
positive constant, or ``inverse temperature". The Gibbs algorithm returns the
distribution 
\begin{equation*}
d\pi _{\mathbf{x}}\left( \omega \right) =Z^{-1}\left( \mathbf{x}\right)
e^{-\beta H\left( \omega ,\mathbf{x}\right) }d\rho \left( \omega \right) 
\text{ where }Z\left( \mathbf{x}\right) =\int_{\Omega }e^{-\beta H\left(
\omega ,\mathbf{x}\right) }d\rho \left( \omega \right) . 
\end{equation*}%
Typically this distribution is the stationary distribution of some
sample-controlled stochastic process characterizing the algorithm. The Gibbs
algorithm plays a role in the simulation of the equilibrium state in
statistical mechanics \citep{Binder 1997} or in non-convex optimization
such as simulated annealing \citep{Kirkpatrick 1983}. There is also some recent attention because $d\pi _{\mathbf{x}%
}$ can be the limiting distribution of randomized algorithms in the training
of deep neural networks \citep{Rakhlin 2017}.

To analyze the Gibbs algorithm we define the function%
\begin{equation}
\label{eq:www}
f:\mathbf{x}\in \mathcal{X}^{n}\mapsto H\left( \cdot,\mathbf{x}\right) \in
L_{\infty }\left( \Omega \right) . 
\end{equation}%
It is easy to verify that this function, which is just a Banach space-valued
sample average, has $\left( 1,0\right) $-weak interactions. Its range is
contained in the unit ball of $L_{\infty }\left( \Omega \right) $.

Related to the Gibbs algorithm is the free energy%
\begin{equation*}
\Lambda \left( \mathbf{x}\right) =\ln Z\left( \mathbf{x}\right) =\ln
\int_{\Omega }e^{-\beta H\left( \omega ,\mathbf{x}\right) }d\rho \left(
\omega \right) , 
\end{equation*}%
which is interesting, because it generates the sample error averaged under
the Gibbs distribution 
\begin{equation*}
\frac{d}{d\beta }\Lambda \left( \mathbf{x}\right) =-\int_{\Omega }H\left(
\omega ,\mathbf{x}\right) d\pi _{\mathbf{x}}\left( \omega \right) . 
\end{equation*}%
Then $\Lambda \left( \mathbf{x}\right) =\Xi \circ f\left( \mathbf{x}%
\right) $ where $\Xi $ is defined as%
\begin{equation*}
\Xi :G\left( .\right) \in L_{\infty }\left( \Omega \right) \mapsto \ln
\int_{\Omega }e^{-\beta G\left( \omega \right) }d\rho \left( \omega \right) .
\end{equation*}%
It is easy to show that $\left\Vert \Xi ^{\prime }\left( G\right)
\right\Vert \leq \beta $ and $\left\Vert \Xi ^{\prime \prime }\left(
G\right) \right\Vert \leq 2\beta ^{2}$ (see appendix Section \ref{Subsection
Gibbs Algorithm computations}). The chain rule Lemma \ref{Lemma chain rule}
then shows that $\Lambda $ has $\left( \beta ,2\beta ^{2}\right) $-weak
interactions, with corresponding consequences for a Bernstein inequality,
normal approximation and estimation of variance for the random free energy $%
\Lambda \left( \mathbf{X}\right) $.

Let $X$ be a random variable with values in $\mathcal{X}$. Then the ``true"
error is given by the function $H_{0}:\omega \mapsto E_{\mathbf{X}}\left[
H\left( \omega ,\mathbf{X}\right) \right] $ and the corresponding Gibbs
measure is%
\begin{equation*}
d\pi \left( \omega \right) =Z^{-1}e^{-\beta H_{0}\left( \omega \right)
}d\rho \left( \omega \right) .
\end{equation*}%
A question of generalization is how much the measures $d\pi _{\mathbf{x}}$
and $d\pi $ differ. We might measure this difference by the Kullback-Leibler
divergence $KL\left( d\pi _{\mathbf{x}},d\pi \right) $ of the two measures.
A mechanical computation using the chain rule (see
appendix Section \ref{Subsection Gibbs Algorithm computations}) shows that the function $%
\mathbf{x}\mapsto KL\left( d\pi _{\mathbf{x}},d\pi \right) $ has $\left(
4\beta ^{2}+2\beta ,12\beta ^{3}+6\beta ^{2}\right) $-weak interactions, 
which again gives useful
information about the random variable $KL\left( d\pi _{\mathbf{X}},d\pi
\right) $.

There is an intuitive parallel to the case of $\ell _{2}$-regularization of
Section \ref{Subsection L_2 regularization}. In both cases the weak
interaction parameters increase, with a corresponding deterioration of
estimation, as we tune more closely to the sample, which for $\ell _{2}$%
-regularization means decreasing $\lambda $ and for the Gibbs algorithm
increasing $\beta $, or lowering the ``temperature". This fits with the general paradigm of regularization.

\section{Summary and some open questions}
    We have shown that functions with weak interactions have tractable statistical properties, and that the class of such functions is quite rich, containing a number of 
well known statistics and other functions relevant to machine learning and statistics. 
    
    Our preliminary survey provides a small probabilistic toolbox which could be used in statistical learning theory. Apart from the application to $\ell_2$-regularized classification, and the analysis of Gibbs algorithms, are there other applications to supervised learning? What is the benefit of finite-sample bounds for normal approximation? Can the empirical Bernstein bound for non-additive functions be used in the analysis of reinforcement learning algorithms, just as its additive counterpart? On the theoretical side, is there a general large deviation principle for weak interactions, in the spirit of Cramer's theorem?
    
\bibliographystyle{natbib}

\section{Appendix}

The appendix contains technical material and a table of notations.
\subsection{Proof of the variance estimation theorem\label{Subsection Proof Variance estimation}}

Define an operator $D^{2}$ on $\mathcal{A}_{n}$ by 
\begin{equation*}
D^{2}f\left( \mathbf{x}\right) =\sum_{k}\left( f\left( \mathbf{x}\right)
-\inf_{y\in \mathcal{X}}S_{y}^{k}f\left( \mathbf{x}\right) \right) ^{2}. 
\end{equation*}%
The proof of Theorem \ref{Theorem Variance Raw} uses the following
concentration inequality which can be found in \citep[Theorem
13]{Maurer 2006} or \cite{Boucheron13}.

\begin{theorem}
\label{Theorem selfbound}Suppose $f:\mathcal{X}^{n}\rightarrow 
\mathbb{R}
$ satisfies for some $a>0$%
\begin{equation}
D^{2}f\left( \mathbf{x}\right) \leq af\left( \mathbf{x}\right) ,\forall 
\mathbf{x}\in \mathcal{X}^{n}\text{,}  \label{selfbound condition}
\end{equation}%
and let $\mathbf{X}=\left( X_{1},...,X_{n}\right) $ be a vector of
independent variables. Then for all $t>0$%
\begin{equation*}
\Pr \left\{ f\left( \mathbf{X}\right) -E\left[ f\right] >t\right\} \leq \exp
\left( \frac{-t^{2}}{2aE\left[ f\left( X\right) \right] +at}\right) \text{.} 
\end{equation*}%
If in addition $f\left( \mathbf{x}\right) -\inf_{y\in \mathcal{X}%
}S_{y}^{k}f\left( \mathbf{x}\right) \leq 1$ for all $k\in \left\{
1,...,n\right\} $ and all $\mathbf{x}\in \mathcal{X}^{n}$ then%
\begin{equation*}
\Pr \left\{ E\left[ f\right] -f\left( \mathbf{X}\right) >t\right\} \leq \exp
\left( \frac{-t^{2}}{2\max \left\{ a,1\right\} E\left[ f\left( X\right) %
\right] }\right) . 
\end{equation*}
\end{theorem}

\begin{corollary}
\label{Corollary selfbound}If $f\in \mathcal{A}_{n}$ satisfies (\ref%
{selfbound condition}) and for some $b>0$ $f\left( \mathbf{x}\right)
-\inf_{y\in \mathcal{X}}S_{y}^{k}f\left( \mathbf{x}\right) \leq b$ for all $%
k\in \left\{ 1,...,n\right\} $ and all $\mathbf{x}\in \mathcal{X}^{n}$ then
for all $\delta >0$ with probability at least $1-\delta $%
\begin{equation*}
\sqrt{f\left( \mathbf{X}\right) }-\sqrt{2a\ln \left( 2/\delta \right) }\leq 
\sqrt{E\left[ f\right] }\leq \sqrt{f\left( \mathbf{X}\right) }+\sqrt{2\max
\left\{ a,b\right\} \ln \left( 2/\delta \right) }. 
\end{equation*}%
For a one-sided bound $2/\delta $ can be replaced by $1/\delta $.
\end{corollary}

\begin{proof}
If $f\left( \mathbf{x}\right) -\inf_{y\in \mathcal{X}}S_{y}^{k}f\left( 
\mathbf{x}\right) \leq b$ then $\left( f\left( \mathbf{x}\right) /b\right)
-\inf_{y\in \mathcal{X}}S_{y}^{k}\left( f\left( \mathbf{x}\right) /b\right)
\leq 1$ and (\ref{selfbound condition}) implies $D^{2}\left( f\left( \mathbf{%
x}\right) /b\right) \leq \left( a/b\right) \left( f\left( \mathbf{x}\right)
/b\right) $, so by the second conclusion of Theorem \ref{Theorem selfbound}%
\begin{eqnarray*}
\Pr \left\{ E\left[ f\right]-f\left( \mathbf{X}\right) >t\right\}  &=&\Pr
\left\{ E\left[ f/b\right]-f\left( \mathbf{X}\right) /b >t/b\right\}  \\
&\leq &\exp \left( \frac{-\left( t/b\right) ^{2}}{2\max \left\{
a/b,1\right\} E\left[ f/b\right] }\right) =\exp \left( \frac{-t^{2}}{2\max
\left\{ a,b\right\} E\left[ f\right] }\right) 
\end{eqnarray*}%
(this is really an alternative formulation of the second conclusion of
Theorem \ref{Theorem selfbound}). Equating the R.H.S. to $\delta $ solving
for $t$ and elementary algebra then give with probability at least $1-\delta 
$ that%
\begin{equation*}
\sqrt{E\left[ f\right] }\leq \sqrt{f\left( \mathbf{X}\right) }+\sqrt{2\max
\left\{ a,b\right\} \ln \left( 1/\delta \right) }.
\end{equation*}%
In a similar way the first conclusion of Theorem \ref{Theorem selfbound}
gives with probability at least $1-\delta $ that%
\begin{equation*}
\sqrt{f\left( \mathbf{X}\right) }-\sqrt{2a\ln \left( 1/\delta \right) }\leq 
\sqrt{E\left[ f\right] }.
\end{equation*}%
A union bound concludes the proof. 
\end{proof}

\begin{proof}{\bf of Theorem \protect\ref{Theorem Variance Raw}}~
First we show that $v_{f}$ is an unbiased estimator for the Efron-Stein
upper bound $E\left[ \Sigma ^{2}\left( f\right) \right] $. Observe that for $%
1\leq i<j\leq n+1$%
\begin{equation*}
E\left[ \left( f\left( S_{-}^{j}\mathbf{X}\right) -f\left(
S_{-}^{j}S_{X_{j}}^{i}\mathbf{X}\right) \right) ^{2}\right] =2E\left[ \sigma
_{i}^{2}\left( f\right) \right] , 
\end{equation*}%
while for $1\leq j<i\leq n+1$%
\begin{equation*}
E\left[ \left( f\left( S_{-}^{j}\mathbf{X}\right) -f\left(
S_{-}^{j}S_{X_{j}}^{i}\mathbf{X}\right) \right) ^{2}\right] =2E\left[ \sigma
_{i-1}^{2}\left( f\right) \right] . 
\end{equation*}%
Thus%
\begin{eqnarray*}
E\left[ v_{f}\right] &=&\frac{1}{2\left( n+1\right) }\sum_{i=1}^{n+1}%
\sum_{j:j\neq i}E\left[ \left( f\left( S_{-}^{j}\mathbf{X}\right) -f\left(
S_{-}^{j}S_{x_{j}}^{i}\mathbf{X}\right) \right) ^{2}\right] \\
&=&\frac{1}{n+1}\left( \sum_{i=2}^{n+1}\sum_{j=1}^{i-1}E\left[ \sigma
_{i-1}^{2}\left( f\right) \right] +\sum_{i=1}^{n}\sum_{j=i+1}^{n+1}E\left[
\sigma _{i}^{2}\left( f\right) \right] \right) \\
&=&\frac{1}{n+1}\sum_{i=1}^{n}\left( n+1\right) E\left[ \sigma
_{i}^{2}\left( f\right) \right] \\
&=&E\left[ \Sigma ^{2}\left( f\right) \right] .
\end{eqnarray*}%
We then apply Corollary \ref{Corollary selfbound} to the function $v_{f}$.
Fix $\mathbf{x}\in \mathcal{X}^{n+1}$, and for each $k\in \left\{
1,...,n+1\right\} $ let $y_{k}:=\arg \min_{y\in \mathcal{X}%
}S_{y}^{k}v_{f}\left( \mathbf{x}\right) $. For $i,j,k\in \left\{
1,...n+1\right\} $ let%
\begin{equation*}
a_{ij}:=f\left( S_{-}^{j}\mathbf{x}\right) -f\left( S_{-}^{j}S_{x_{j}}^{i}%
\mathbf{x}\right) \text{ and }a_{ijk}:=f\left( S_{-}^{j}S_{y_{k}}^{k}\mathbf{%
x}\right) -f\left( S_{-}^{j}S_{x_{j}}^{i}S_{y_{k}}^{k}\mathbf{x}\right) . 
\end{equation*}%
Then%
\begin{equation}
v_{f}\left( \mathbf{x}\right) =\frac{1}{2\left( n+1\right) }%
\sum_{i}\sum_{j:j\neq i}a_{ij}^{2}.  \label{reference inequality}
\end{equation}%
Observe that $\left\vert a_{ij}\right\vert ,\left\vert a_{ijk}\right\vert
\leq M\left( f\right) $ and that $J( f\circ S_{-}^{j}) =J(
f) $ so for $j\neq i\neq k\neq j$ $\left\vert
a_{ij}-a_{ijk}\right\vert \leq J\left( f\right) /n.$ Also the replacement of
a component, which is then deleted, has no effect, so%
\begin{equation*}
a_{ikk}=f\left( S_{-}^{k}S_{y_{k}}^{k}\mathbf{x}\right) -f\left(
S_{-}^{k}S_{x_{k}}^{i}S_{y_{k}}^{k}\mathbf{x}\right) =f\left( S_{-}^{k}%
\mathbf{x}\right) -f\left( S_{-}^{k}S_{x_{k}}^{i}\mathbf{x}\right) =a_{ik}. 
\end{equation*}%
With reference to a fixed index $k\in \left\{ 1,...,n+1\right\} $ we can
write 
\begin{equation}
v_{f}\left( \mathbf{x}\right) =\frac{1}{2\left( n+1\right) }\left(
\sum_{j:j\neq k}a_{kj}^{2}+\sum_{i:i\neq k}a_{ik}^{2}+\sum_{i,j:i\neq
j\wedge k\notin \left\{ i,j\right\} }a_{ij}^{2}\right) .  \notag
\end{equation}%
In the expression for $v_{f}\left( \mathbf{x}\right)
-S_{y_{k}}^{k}v_{f}\left( \mathbf{x}\right) $ the second sum in the last
expression cancels, so%
\begin{eqnarray}
0 &\leq &v_{f}\left( \mathbf{x}\right) -S_{y_{k}}^{k}v_{f}\left( \mathbf{x}%
\right)  \notag \\
&\leq &\frac{1}{2\left( n+1\right) }\left( \sum_{j:j\neq
k}a_{kj}^{2}+\sum_{i,j:i\neq j\wedge k\notin \left\{ i,j\right\} \ }\left(
a_{ij}^{2}-a_{ijk}^{2}\right) \right)  \notag \\
&=&\frac{1}{2\left( n+1\right) }\left( \sum_{j:j\neq
k}a_{kj}^{2}+\sum_{i,j:i\neq j\wedge k\notin \left\{ i,j\right\} \ }\left(
a_{ij}-a_{ijk}\right) \left( a_{ij}+a_{ijk}\right) \right)  \notag \\
&\leq &M\left( f\right) ^{2}/2+M\left( f\right) J\left( f\right) .
\label{b-bound}
\end{eqnarray}%
We square and sum over $k$, and use $\left( s+t\right) ^{2}\leq
2s^{2}+2t^{2} $ for real $s,t$, and Cauchy-Schwarz to obtain%
\begin{eqnarray*}
D^{2}v_{f}\left( \mathbf{x}\right) &=&\sum_{k}\left( v_{f}\left( \mathbf{x}%
\right) -S_{y_{k}}^{k}v_{f}\left( \mathbf{x}\right) \right) ^{2} \\
&\leq &\frac{1}{2\left( n+1\right) ^{2}}\sum_{k}\left( \sum_{j:j\neq
k}a_{kj}^{2}\right) ^{2}+ \\
&&+\frac{1}{2\left( n+1\right) ^{2}}\sum_{k}\sum_{i,j:i\neq j\wedge k\notin
\left\{ i,j\right\} \ }\left( a_{ij}-a_{ijk}\right) ^{2}\sum_{i,j:i\neq
j\wedge k\notin \left\{ i,j\right\} }\left( a_{ij}+a_{ijk}\right) ^{2} \\
&=&:A+B.
\end{eqnarray*}%
We treat the two terms in turn. For $A$ we get%
\begin{eqnarray*}
A &=&\frac{1}{2\left( n+1\right) ^{2}}\sum_{k}\left( \sum_{j:j\neq
k}a_{kj}^{2}\right) ^{2} \\
&\leq &\frac{M\left( f\right) ^{2}}{2\left( n+1\right) }\sum_{k}\sum_{j:j%
\neq k}a_{kj}^{2}=M\left( f\right) ^{2}v_{f}\left( \mathbf{x}\right) .
\end{eqnarray*}%
For $B$ we again use $a_{ij}-a_{ijk}\leq J\left( f\right) /n$ and $\left(
s+t\right) ^{2}\leq 2s^{2}+2t^{2}$ to get%
\begin{eqnarray*}
B &=&\frac{1}{2\left( n+1\right) ^{2}}\sum_{k}\sum_{i,j:i\neq j\wedge
k\notin \left\{ i,j\right\} \ }\left( a_{ij}-a_{ijk}\right)
^{2}\sum_{i,j:i\neq j\wedge k\notin \left\{ i,j\right\} }\left(
a_{ij}+a_{ijk}\right) ^{2} \\
&\leq &\frac{2J\left( f\right) ^{2}}{n+1}\sum_{k}\left( \frac{1}{2\left(
n+1\right) }\sum_{i,j:i\neq j\wedge k\notin \left\{ i,j\right\} }a_{ij}^{2}+%
\frac{1}{2\left( n+1\right) }\sum_{i,j:i\neq j\wedge k\notin \left\{
i,j\right\} }a_{ijk}^{2}\right)
\end{eqnarray*}%
But by (\ref{reference inequality}) for every $k\in \left\{
1,...,n+1\right\} $%
\begin{equation*}
\frac{1}{2\left( n+1\right) }\sum_{i,j:i\neq j\wedge k\notin \left\{
i,j\right\} }a_{ij}^{2}\leq v_{f}\left( \mathbf{x}\right) 
\end{equation*}%
and also, by the definition of $y_{k}$, 
\begin{equation*}
\frac{1}{2\left( n+1\right) }\sum_{i,j:i\neq j\wedge k\notin \left\{
i,j\right\} }a_{ijk}^{2}\leq S_{y_{k}}^{k}v_{f}\left( \mathbf{x}\right) \leq
v_{f}\left( \mathbf{x}\right) . 
\end{equation*}%
It follows that $B\leq 4J\left( f\right) ^{2}v_{f}\left( \mathbf{x}\right) $
and 
\begin{equation*}
D^{2}v_{f}\left( \mathbf{x}\right) \leq \left( M\left( f\right)
^{2}+4J\left( f\right) ^{2}\right) v_{f}\left( x\right) . 
\end{equation*}%
Together with (\ref{b-bound}) this can be used in Corollary \ref{Corollary
selfbound}. Since 
\begin{equation*}
M\left( f\right) ^{2}/2+M\left( f\right) J\left( f\right) \leq \frac{1}{2}%
\left( M\left( f\right) +J\left( f\right) \right) ^{2}\leq M\left( f\right)
^{2}+J\left( f\right) ^{2}\leq M\left( f\right) ^{2}+4J\left( f\right) ^{2}, 
\end{equation*}%
the corollary gives us for any $\delta >0$ with probability at least $%
1-\delta $%
\begin{equation*}
\left\vert \sqrt{E\left[ \Sigma ^{2}\left( f\right) \right] }-\sqrt{%
v_{f}\left( \mathbf{X}\right) }\right\vert \leq \sqrt{\left( 2M\left(
f\right) ^{2}+8J\left( f\right) ^{2}\right) \ln \left( 2/\delta \right) }. 
\end{equation*}%

\end{proof}

\subsection{Proof of the normal approximation theorem\label{Subsection Proof normal approximation}%
}

To prove Theorem \ref{Theorem normal} we use a result of Chatterjee (\cite%
{Chatterjee 2008}, Theorem 2.2), for which we need extra notation. Let $%
\mathbf{X}^{\prime }=\left( X_{1}^{\prime },...,X_{n}^{\prime }\right) $ be
an independent copy of $\mathbf{X}=\left( X_{1},...,X_{n}\right) $. For a
proper subset $A\subsetneqq \left\{ 1,...,n\right\} $ define the vector $%
\mathbf{X}^{A}=\mathbf{X}^{A}\left( \mathbf{X},\mathbf{X}^{\prime }\right) $
to be 
\begin{equation*}
X_{i}^{A}=\left\{ 
\begin{array}{ccc}
X_{i}^{\prime } & \text{if} & i\in A \\ 
X_{i} & \text{if} & i\notin A%
\end{array}%
\right. . 
\end{equation*}%
For $A\subsetneqq \left\{ 1,...,n\right\} $ define the random variables%
\begin{eqnarray*}
T_{A} &=&T_{A}\left( \mathbf{X},\mathbf{X}^{\prime }\right) =\sum_{j\notin
A}\left( D_{X_{j},X_{j}^{\prime }}^{j}f\left( \mathbf{X}\right) \right)
\left( D_{X_{j},X_{j}^{\prime }}^{j}f\left( \mathbf{X}^{A}\right) \right) \\
\text{and }T &=&T\left( \mathbf{X},\mathbf{X}^{\prime }\right) =\frac{1}{2}%
\sum_{A\subsetneqq \left\{ 1,...,n\right\} }\frac{T_{A}}{\binom{n}{%
\left\vert A\right\vert }\left( n-\left\vert A\right\vert \right) }.
\end{eqnarray*}

\begin{theorem}
\label{Theorem Chatterjee} (Chatterjee) Let $f:\mathcal{X}^{n}\rightarrow 
\mathbb{R}
$ and suppose $E\left[ f\right] =0$ and $\sigma ^{2}\left( f\right) <\infty
. $ Then 
\begin{equation}
d_{\mathcal{N}}\left( f\left( \mathbf{X}\right) \right) \leq \frac{\sqrt{%
\sigma ^{2}\left( E\left[ T|\mathbf{X}\right] \right) }}{\sigma ^{2}\left(
f\right) }+\frac{1}{2\sigma ^{3}\left( f\right) }\sum_{j=1}^{n}E\left[
\left\vert D_{X_{j},X_{j}^{\prime }}^{j}f\left( \mathbf{X}\right)
\right\vert ^{3}\right] .  \label{Chatterjee bound}
\end{equation}
\end{theorem}

\begin{proof}{\bf of Theorem \protect\ref{Theorem normal}}
Both sides of the inequality we wish to prove do not change when a
constant is added to $f$. We can therefore assume $E\left[ f\right] =0$ and
use Chatterjee's theorem. We can bound the second term in (\ref{Chatterjee bound}) immediately by $%
nM\left( f\right) ^{3}/\left( 2\sigma ^{3}\left( f\right) \right) $, so the
main work is in bounding $\sqrt{\sigma ^{2}\left( E\left[ T|\mathbf{X}\right]
\right) }$. By the $L_{2}$-triangle inequality (Minkovsky-inequality) we have%
\begin{eqnarray*}
\sqrt{\sigma ^{2}\left( E\left[ T|\mathbf{X}\right] \right) } &\leq &\frac{1%
}{2}\sum_{A\subset \left\{ 1,...,n\right\} }\frac{\sqrt{\sigma ^{2}\left( E%
\left[ T_{A}|\mathbf{X}\right] \right) }}{\binom{n}{\left\vert A\right\vert }%
\left( n-\left\vert A\right\vert \right) } \\
&\leq &\frac{1}{2}\sum_{A\subset \left\{ 1,...,n\right\} }\frac{\sqrt{E\left[
\sigma ^{2}\left( T_{A}|\mathbf{X}^{\prime }\right) \right] }}{\binom{n}{%
\left\vert A\right\vert }\left( n-\left\vert A\right\vert \right) },
\end{eqnarray*}%
where we used Lemma 4.4 in \cite{Chatterjee 2008} for the second inequality.
So we first need to bound $E\left[ \sigma ^{2}\left( T_{A}|\mathbf{X}%
^{\prime }\right) \right] $ for fixed $A\subsetneqq \left\{ 1,...,n\right\} $%
. This is done with the Efron Stein inequality \cite{Efron 1981}, which
gives 
\begin{equation*}
E\left[ \sigma ^{2}\left( T_{A}|\mathbf{X}^{\prime }\right) \right] \leq 
\frac{1}{2}E\left[ \sum_{i=1}^{n}\left( T_{A}\left( \mathbf{X},\mathbf{X}%
^{\prime }\right) -S_{i}^{\prime \prime }T_{A}\left( \mathbf{X},\mathbf{X}%
^{\prime }\right) \right) ^{2}|\mathbf{X}^{\prime }\right] , 
\end{equation*}%
where $\mathbf{X}^{\prime \prime }$ is yet another independent copy of $%
\mathbf{X}$, and the operator $S_{i}^{\prime \prime }$ acts on functions of $%
2n$ variables and substitutes every occurence of $X_{i}\,\ $by $%
X_{i}^{\prime \prime }$%
\begin{equation*}
\left( S_{i}^{\prime \prime }F\right) \left( \mathbf{X},\mathbf{X}^{\prime
}\right) =F\left( X_{1},...,X_{i-1},X_{i}^{\prime \prime },X_{i+1},...,X_{n},%
\mathbf{X}^{\prime }\right) . 
\end{equation*}%
Now let $V_{j}:=D_{X_{j},X_{j}^{\prime }}^{j}f\left( \mathbf{X}\right) $, $%
W_{j}:=D_{X_{j},X_{j}^{\prime }}^{j}f\left( \mathbf{X}^{A}\right) $, $%
V_{ij}:=S_{i}^{\prime \prime }V_{j}=S_{i}^{^{\prime \prime
}}D_{X_{j},X_{j}^{\prime }}^{j}f\left( \mathbf{X}\right) $ and $%
W_{ij}:=S_{i}^{\prime \prime }W_{j}=S_{i}^{\prime \prime
}D_{X_{j},X_{j}^{\prime }}^{j}f\left( \mathbf{X}^{A}\right) $. Observe that
all of $V_{j}$, $W_{j}$, $V_{ij}$ and $W_{ij}$ have absolute value bounded
by $M\left( f\right) $, and that for $i\neq j$ 
\begin{equation*}
\left\vert V_{j}-V_{ij}\right\vert \leq J\left( f\right) /n\text{ and }%
\left\vert W_{j}-W_{ij}\right\vert \leq J\left( f\right) /n. 
\end{equation*}%
Then%
\begin{align*}
& \sum_{i=1}^{n}\left( T_{A}\left( \mathbf{X},\mathbf{X}^{\prime }\right)
-S_{i}^{\prime \prime }T_{A}\left( \mathbf{X},\mathbf{X}^{\prime }\right)
\right) ^{2} \\
& =\sum_{i=1}^{n}\left( \sum_{j\notin A}V_{j}W_{j}-V_{ij}W_{ij}\right) ^{2}
\\
& =\sum_{i=1}^{n}\left( \sum_{j\notin A,j\neq i}\left(
V_{j}W_{j}-V_{ij}W_{ij}\right) +1_{\left\{ i\notin A\right\} }\left(
V_{i}W_{i}-V_{ii}W_{ii}\right) \right) ^{2} \\
& \leq 2\sum_{i=1}^{n}\left( \sum_{j\notin A,j\neq
i}V_{j}W_{j}-V_{ij}W_{ij}\right) ^{2}+2\sum_{i\notin A}\left(
V_{i}W_{i}-V_{ii}W_{ii}\right) ^{2} \\
& \leq 2\sum_{i=1}^{n}\left( \sum_{j\notin A,j\neq
i}V_{j}W_{j}-V_{ij}W_{ij}\right) ^{2}+8nM\left( f\right) ^{4}.
\end{align*}%
Now, using Cauchy Schwarz,%
\begin{align*}
& 2\sum_{i=1}^{n}\left( \sum_{j\notin A,j\neq
i}V_{j}W_{j}-V_{ij}W_{ij}\right) ^{2} \\
& =2\sum_{i=1}^{n}\left( \sum_{j\notin A,j\neq i}\left( V_{j}-V_{ij}\right)
W_{j}+V_{ij}\left( W_{j}-W_{ij}\right) \right) ^{2} \\
& \leq 4\sum_{i=1}^{n}\left( \sum_{j\notin A,j\neq i}\left(
V_{j}-V_{ij}\right) W_{j}\right) ^{2}+4\sum_{i=1}^{n}\left( \sum_{j\notin
A,j\neq i}V_{ij}\left( W_{j}-W_{ij}\right) \right) ^{2} \\
& \leq 4\sum_{i=1}^{n}\sum_{j\notin A,j\neq i}\left( V_{j}-V_{ij}\right)
^{2}\sum_{j\notin A,j\neq i}W_{j}^{2}+4\sum_{i=1}^{n}\sum_{j\notin A,j\neq
i}V_{ij}^{2}\sum_{j\notin A,j\neq i}\left( W_{j}-W_{ij}\right) ^{2} \\
& \leq 8\sum_{i=1}^{n}\sum_{j\notin A,j\neq i}\frac{J\left( f\right) ^{2}}{%
n^{2}}\sum_{j\notin A,j\neq i}M\left( f\right) ^{2} \\
& \leq 8nM\left( f\right) ^{2}J\left( f\right) ^{2}.
\end{align*}%
Putting the chains of inequalities together and using $\sqrt{s+t}\leq \sqrt{s%
}+\sqrt{t}$ we conclude that%
\begin{equation*}
\sqrt{E\left[ \sigma ^{2}\left( T_{A}|\mathbf{X}^{\prime }\right) \right] }%
\leq 2\sqrt{n}M\left( f\right) \left( M\left( f\right) +J\left( f\right)
\right) . 
\end{equation*}%
Thus%
\begin{eqnarray*}
\sqrt{\sigma ^{2}\left( E\left[ T|X\right] \right) } &\leq &\frac{1}{2}%
\sum_{A\subsetneqq \left\{ 1,...,n\right\} }\frac{\sqrt{E\left[ \sigma
^{2}\left( T_{A}|X^{\prime }\right) \right] }}{\binom{n}{\left\vert
A\right\vert }\left( n-\left\vert A\right\vert \right) } \\
&\leq &\sqrt{n}M\left( f\right) \left( J\left( f\right) +M\left( f\right)
\right) \sum_{k=1}^{n-1}\sum_{A:\left\vert A\right\vert =k}\frac{\binom{n-1}{%
k}}{\binom{n}{k}\left( n-k\right) } \\
&=&\sqrt{n}M\left( f\right) \left( J\left( f\right) +M\left( f\right) \right)
\end{eqnarray*}%
By Theorem \ref{Theorem Chatterjee} and the bound on the last term of (\ref%
{Chatterjee bound})%
\begin{equation*}
d_{\mathcal{N}}\left( f\left( \mathbf{X}\right) \right) \leq \frac{\sqrt{n}%
M\left( f\right) \left( J\left( f\right) +M\left( f\right) \right) }{\sigma
^{2}\left( f\right) }+\frac{nM\left( f\right) ^{3}}{2\sigma ^{3}\left(
f\right) }. 
\end{equation*}
\end{proof}

\subsection{Lipschitz L-statistics revisited\label{Subsection Argument L-statistics}}
We show that the Lipschitz L-statistics of Section \ref{Subsection Lipschitz
	L-statistics} have $\left( \left\Vert F\right\Vert _{\infty },\left\Vert
F\right\Vert _{Lip}\right) $-weak interactions. For $\alpha ,\beta \in 
\mathbb{R}
$ let $\left[ \left[ \alpha ,\beta \right] \right] $ be the interval $\left[
\min \left\{ \alpha ,\beta \right\} ,\max \left\{ \alpha ,\beta \right\} %
\right] $. That $f$ as defined in equation (\ref{L-statistic}) has $\left(
\left\Vert F\right\Vert _{\infty },\left\Vert F\right\Vert _{Lip}\right) $%
-weak interactions is clearly implied by

\begin{theorem}
	\label{Theorem LstatWeak}With $f$ as (\ref{L-statistic}) we have%
	\begin{eqnarray}
	\left\vert D_{y,y^{\prime }}^{k}f\left( \mathbf{x}\right) \right\vert  &\leq
	&\frac{\left\Vert F\right\Vert _{\infty }\text{diam}\left( \left[ \left[
		y,y^{\prime }\right] \right] \right) }{n}  \label{Lstat1 conditio} \\
	\left\vert D_{z,z^{\prime }}^{l}D_{y,y^{\prime }}^{k}f\left( \mathbf{x}%
	\right) \right\vert  &\leq &\frac{\left\Vert F\right\Vert _{Lip}\text{diam}%
		\left( \left[ \left[ z,z^{\prime }\right] \right] \cap \left[ \left[
		y,y^{\prime }\right] \right] \right) }{n^{2}}  \label{Lstat2 conditio}
	\end{eqnarray}%
	for any $\mathbf{x}\in \left[ 0,1\right] ^{n},$\textbf{\ }all $k\neq l$ and
	all $y,y^{\prime },z,z^{\prime }\in \left[ 0,1\right] $.
\end{theorem}

\begin{proof}
	Suppose we can prove the inequalities (\ref{Lstat1 conditio}) and (\ref%
	{Lstat2 conditio}) for all $\mathbf{x}\in \left[ 0,1\right] ^{n}$ and all $%
	k\neq l$ and in the three cases
	
	\begin{tabular}{l|l|l}
		a & $z^{\prime }\leq z<y^{\prime }\leq y$ & $\left[ \left[ z,z^{\prime }%
		\right] \right] \cap \left[ \left[ y,y^{\prime }\right] \right] =\emptyset $%
		, non-intersection \\ 
		b & $z^{\prime }\leq y^{\prime }\leq y\leq z$ & $\left[ \left[ y,y^{\prime }%
		\right] \right] \subseteq \left[ \left[ z,z^{\prime }\right] \right] $,
		inclusion \\ 
		c & $z^{\prime }\leq y^{\prime }\leq z\leq y$ & partial intersection.%
	\end{tabular}
	
	The right collumn above enumerates all possible relationships between $\left[
	\left[ z,z^{\prime }\right] \right] $ and $\left[ \left[ y,y^{\prime }\right]
	\right] $. Then, as (\ref{Lstat1 conditio}) and (\ref{Lstat2 conditio}) are
	invariant under the exchanges of $k\leftrightarrow l$, $z\leftrightarrow
	z^{\prime }$ and $y\leftrightarrow y^{\prime }$, we have proven these
	inequalities for all possible orderings of $z,z^{\prime },y$ and $y^{\prime }
	$. It therefore suffices to prove the above inequality in the three cases a,
	b and c.
	
	To further simplify the problem we introduce the vector $\mathbf{\hat{x}}\in %
	\left[ 0,1\right] ^{n}$ defined by 
	\[
	\hat{x}_{i}=\left( S_{z^{\prime }}^{l}S_{y^{\prime }}^{k}\mathbf{x}\right)
	_{\left( i\right) }\text{.}
	\]%
	Then $\mathbf{\hat{x}}$ is already ordered, and there are $\hat{l}$ and $%
	\hat{k}$ in $\left\{ 1,...,n\right\} $ such that $\hat{l}\neq \hat{k}$ and $%
	\hat{x}_{\hat{l}}=z^{\prime }$ and $\hat{x}_{\hat{k}}=y^{\prime }$. Write $%
	F_{i}=F\left( i/n\right) $, so that $\left\vert F_{i}\right\vert \leq
	\left\Vert F\right\Vert _{\infty }$ and $\left\vert F_{i}-F_{i-1}\right\vert
	\leq \left\Vert F\right\Vert _{Lip}/n$. Transcribing to the new variables
	and omitting the "\symbol{94}"-symbols, it becomes appearant that we have to
	prove the inequalities%
	\begin{eqnarray*}
		A &:&=\left\vert \sum_{i=1}^{n}F_{i}\left( \mathbf{x}_{\left( i\right) }%
		\mathbf{-}\left( S_{y}^{k}\mathbf{x}\right) _{\left( i\right) }\right)
		\right\vert \leq \left\Vert F\right\Vert _{\infty }\text{diam}\left( \left[ %
		\left[ y,y^{\prime }\right] \right] \right) \text{ and} \\
		B &:&=\left\vert \sum_{i=1}^{n}F_{i}\left( \mathbf{x}_{\left( i\right) }%
		\mathbf{-}\left( S_{y}^{k}\mathbf{x}\right) _{\left( i\right) }-\left(
		\left( S_{z}^{l}\mathbf{x}\right) _{\left( i\right) }\mathbf{-}\left(
		S_{z}^{l}S_{y}^{k}\mathbf{x}\right) _{\left( i\right) }\right) \right)
		\right\vert  \\
		&\leq &\frac{\left\Vert F\right\Vert _{Lip}\text{diam}\left( \left[ \left[
			z,z^{\prime }\right] \right] \cap \left[ \left[ y,y^{\prime }\right] \right]
			\right) }{n}
	\end{eqnarray*}%
	for all $\mathbf{x}\in \left[ 0,1\right] ^{n}$, which are already ordered
	with $x_{i}\leq x_{i+1}$, and all $k\neq l$ and in the three cases\newline
	\begin{tabular}{l|l}
		a & \multicolumn{1}{|l}{$x_{l}\leq z<x_{k}\leq y$} \\ 
		b & \multicolumn{1}{|l}{$x_{l}\leq x_{k}\leq y\leq z$} \\ 
		c & \multicolumn{1}{|l}{$x_{l}\leq x_{k}\leq z\leq y$}%
	\end{tabular}%
	.\newline
	We let $p,q\in \left\{ 1,...,n\right\} $ be such that%
	\[
	\left( S_{y}^{k}\mathbf{x}\right) _{\left( p\right) }=y\text{ and }\left(
	S_{z}^{l}\mathbf{x}\right) _{\left( q\right) }=z\text{.}
	\]%
	The effect which modifying an argument has on the order statistic is a shift
	and the replacement of a boundary term. For $x_{k}\leq y$ we have%
	\[
	\left( S_{y}^{k}\mathbf{x}\right) _{\left( i\right) }=\left\{ 
	\begin{array}{ccc}
	x_{i} & \text{if} & i\notin \left\{ k,...,p\right\}  \\ 
	x_{i+1} & \text{if} & i\in \left\{ k,...,p-1\right\}  \\ 
	y & \text{if} & i=p%
	\end{array}%
	\right. \text{.}
	\]%
	It follows that in all cases a, b and c 
	\begin{eqnarray*}
		A &=&\left\vert \sum_{i=k}^{p-2}F_{i}\left( x_{i}-x_{i+1}\right)
		+F_{p-1}\left( x_{p-1}-y\right) \right\vert  \\
		&\leq &\left\Vert F\right\Vert _{\infty }\left( \sum_{i=k}^{p-2}\left\vert
		x_{i}-x_{i+1}\right\vert +\left\vert x_{p-1}-y\right\vert \right) \leq
		\left\Vert F\right\Vert _{\infty }\left( y-x_{k}\right) ,
	\end{eqnarray*}%
	which gives the bound on $A$ and therefore (\ref{Lstat1 conditio}).
	
	For the second inequality it is easy to see that $B=0$ whenever $\left[ %
	\left[ x_{k},y\right] \right] $ and $\left[ \left[ x_{l},z\right] \right] $
	don't intersect, as in \textbf{Case a}, so we consider only the cases b and
	c.
	
	\textbf{Case b} (inclusion, $x_{l}\leq x_{k}\leq y\leq z$). By partial
	summation we get%
	\begin{eqnarray*}
		B &=&\left\vert \sum_{i=k}^{p-1}\left( F_{i}-F_{i-1}\right) \left(
		x_{i}-x_{i+1}\right) +\left( F_{p}-F_{p-1}\right) \left( x_{p}-y\right)
		\right\vert  \\
		&\leq &\frac{\left\Vert F\right\Vert _{Lip}}{n}\sum_{i=k}^{p-1}\left\vert
		x_{i}-x_{i+1}\right\vert +\frac{\left\Vert F\right\Vert _{Lip}}{n}\left\vert
		x_{p}-y\right\vert  \\
		&=&\frac{\left\Vert F\right\Vert _{Lip}}{n}\left( y-x_{k}\right) .
	\end{eqnarray*}%
	The general principle here is partial summation and the fact that the sum of
	absolute differences always collapses to the diameter of an interval because
	of the ordering. 
	
	\textbf{Case c}, (partial intersection, $x_{l}\leq x_{k}\leq z\leq y$). 
	\begin{eqnarray*}
		B &=&\left\vert \sum_{i=k}^{q-1}\left( F_{i}-F_{i-1}\right) \left(
		x_{i}-x_{i+1}\right) +\left( F_{q}-F_{q-1}\right) \left( x_{q}-z\right)
		\right\vert  \\
		&\leq &\frac{\left\Vert F\right\Vert _{Lip}}{n}\sum_{i=k}^{q-1}\left\vert
		x_{i}-x_{i+1}\right\vert +\frac{\left\Vert F\right\Vert _{Lip}}{n}\left\vert
		x_{q}-z\right\vert  \\
		&=&\frac{\left\Vert F\right\Vert _{Lip}}{n}\left( z-x_{k}\right) .
	\end{eqnarray*}
\end{proof}

\subsection{Proof of the chain rule\label{Subsection Proof Chainrule}}

\begin{proof}{\bf of Lemma \protect\ref{Lemma chain rule}}
Take arbitrary $\mathbf{x}\in \mathcal{X}^{n}$, $y,y^{\prime },z,z^{\prime
}\in \mathcal{X}$ and any $k,l,k\neq l$. Define a linear function $h:\left[
0,1\right] \rightarrow U$ by%
\begin{equation*}
h\left( t\right) =tf\left( S_{y}^{k}\mathbf{x}\right) +\left( 1-t\right)
f\left( S_{y^{\prime }}^{k}\mathbf{x}\right) . 
\end{equation*}%
Then $h^{\prime }\left( t\right) =D_{y,y^{\prime }}^{k}f\left( \mathbf{x}%
\right) $ and 
\begin{eqnarray*}
D_{y,y^{\prime }}^{k}F\circ f\left( \mathbf{x}\right) &=&F\left( h\left(
1\right) \right) -F\left( h\left( 0\right) \right) =\int_{0}^{1}F^{\prime
}\left( h\left( t\right) \right) h^{\prime }\left( t\right) dt \\
&\leq &\int_{0}^{1}\left\Vert F^{\prime }\left( h\left( t\right) \right)
\right\Vert \left\Vert D_{y,y^{\prime }}^{k}f\left( \mathbf{x}\right)
\right\Vert dt\leq \sup_{v\in U}\left\Vert F^{\prime }\left( v\right)
\right\Vert M\left( f\right) .
\end{eqnarray*}%
This proves the first inequality. For the bound on $J$ define a bilinear
function $g:\left[ 0,1\right] \times \left[ 0,1\right] \rightarrow U$ by%
\begin{equation*}
g\left( s,t\right) =stf\left( S_{z}^{l}S_{y}^{k}\mathbf{x}\right) +s\left(
1-t\right) f\left( S_{z}^{l}S_{y^{\prime }}^{k}\mathbf{x}\right) +t\left(
1-s\right) f\left( S_{z^{\prime }}^{l}S_{y}^{k}\mathbf{x}\right) +\left(
1-s\right) \left( 1-t\right) f\left( S_{z^{\prime }}^{l}S_{y^{\prime }}^{k}%
\mathbf{x}\right) . 
\end{equation*}%
Then $\left\Vert \frac{\partial }{\partial t}g\left( s,t\right) \right\Vert
=\left\Vert sD_{y,y^{\prime }}^{k}f\left( S_{z}^{l}\mathbf{x}\right) +\left(
1-s\right) D_{y,y^{\prime }}^{k}f\left( S_{z^{\prime }}^{l}\mathbf{x}\right)
\right\Vert \leq M\left( f\right) $ and similarly $\left\Vert \frac{\partial 
}{\partial t}g\left( s,t\right) \right\Vert \leq M\left( f\right) $ and also 
$\left\Vert \frac{\partial ^{2}}{\partial s\partial t}g\left( s,t\right)
\right\Vert =\left\Vert D_{zz^{\prime }}^{l}D_{yy^{\prime }}^{k}\right\Vert
\leq J\left( f\right) /n$. Thus 
\begin{eqnarray*}
\left\vert \frac{\partial ^{2}}{\partial s\partial t}F\left( g\left(
s,t\right) \right) \right\vert &=&\left\vert F^{\prime \prime }\left(
g\left( s,t\right) \right) \frac{\partial }{\partial t}g\left( s,t\right) 
\frac{\partial }{\partial s}g\left( s,t\right) +F^{\prime }\left( g\left(
s,t\right) \right) \frac{\partial ^{2}}{\partial s\partial t}g\left(
s,t\right) \right\vert \\
&\leq &\left\Vert F^{\prime \prime }\left( g\left( s,t\right) \right)
\right\Vert \left\Vert \frac{\partial }{\partial t}g\left( s,t\right)
\right\Vert \left\Vert \frac{\partial }{\partial s}g\left( s,t\right)
\right\Vert +\left\Vert F^{\prime }\left( g\left( s,t\right) \right)
\right\Vert \left\Vert \frac{\partial ^{2}}{\partial s\partial t}g\left(
s,t\right) \right\Vert \\
&\leq &\left\Vert F^{\prime \prime }\left( g\left( s,t\right) \right)
\right\Vert M\left( f\right) ^{2}+\left\Vert F^{\prime }\left( g\left(
s,t\right) \right) \right\Vert J\left( f\right) /n
\end{eqnarray*}%
So that%
\begin{eqnarray*}
D_{zz^{\prime }}^{l}D_{yy^{\prime }}^{k}F\circ f\left( \mathbf{x}\right)
&=&F\left( g\left( 1,1\right) \right) -F\left( g\left( 1,0\right) \right)
-\left( F\left( g\left( 0,1\right) \right) -F\left( g\left( 0,0\right)
\right) \right) \\
&=&\int_{0}^{1}\int_{0}^{1}\frac{\partial ^{2}}{\partial s\partial t}F\left(
g\left( s,t\right) \right) ds~dt \\
&\leq &\left\Vert F^{\prime \prime }\left( g\left( s,t\right) \right)
\right\Vert M\left( f\right) ^{2}+\left\Vert F^{\prime }\left( g\left(
s,t\right) \right) \right\Vert J\left( f\right) /n.
\end{eqnarray*}%
The second inequality follows.
\end{proof}

\subsection{The Gibbs algorithm\label{Subsection Gibbs Algorithm
computations}}

We substantiate the claims in Section \ref{Subsection Gibbs algorithm}. For $%
G\in L_{\infty }\left( \Omega \right) $ define 
\begin{equation*}
Z\left( G\right) =\int_{\Omega }e^{-\beta G\left( \omega \right) }d\rho
\left( \omega \right) 
\end{equation*}%
and an expectation functional%
\begin{equation*}
E_{G}\left[ h\right] :=Z\left( G\right) ^{-1}\int_{\Omega }h\left( \omega
\right) e^{-\beta G\left( \omega \right) }d\rho \left( \omega \right) \text{
for }h\in L_{\infty }\left( \Omega \right) .
\end{equation*}%
Then%
\begin{eqnarray*}
KL\left( d\pi _{\mathbf{x}},d\pi \right)  &=&E_{H\left( \cdot,\mathbf{x}\right) }%
\left[ \ln \left( \frac{Z^{-1}\left( \mathbf{x}\right) e^{-\beta H\left(
\omega ,\mathbf{x}\right) }}{Z_{0}^{-1}e^{-\beta H_{0}\left( \omega \right) }%
}\right) \right]  \\
&=&F\circ f\left( \mathbf{x}\right) ,
\end{eqnarray*}%
where $f$ is defined in equation \eqref{eq:www} and $F$ the real function defined on the unit ball $\mathbb{B}_{1}$ of $%
L_{\infty }\left( \Omega \right) $ by%
\begin{equation*}
F\left( G\right) :=E_{G}\left[ \ln \left( \frac{Z\left( G\right)
^{-1}e^{-\beta G}}{Z_{0}^{-1}e^{-\beta H_{0}}}\right) \right] =\beta E_{G}%
\left[ H_{0}-G\right] -\ln Z\left( G\right) +\ln Z_{0}.
\end{equation*}%
To apply the chain rule we have to bound the derivatives of $F=\beta \left(
\Psi -\Phi \right) -\Xi +\ln Z_{0}$, where $\Psi ,\Phi ,\Xi :\mathbb{B}%
_{1}\rightarrow 
\mathbb{R}
$ are the functions%
\begin{equation*}
\Xi \left( G\right) =\ln Z\left( G\right) \text{, }\Phi \left( G\right)
=E_{G}\left[ G\right] \text{ and}\Psi \left( G\right) =E_{G}\left[ H_{0}%
\right] \text{.}
\end{equation*}%
Differentiating we find 
\begin{eqnarray*}
\Xi ^{\prime }\left( G\right) \left[ u\right]  &=&-\beta E_{G}\left[ u\right]
\\
\Xi ^{\prime \prime }\left( G\right) \left[ u\right] \left[ v\right] 
&=&\beta ^{2}\left( E_{G}\left[ uv\right] -E_{G}\left[ u\right] E_{G}\left[ v%
\right] \right) 
\end{eqnarray*}%
so that $\left\Vert \Xi ^{\prime }\right\Vert \leq \beta $ and $\left\Vert
\Xi ^{\prime \prime }\right\Vert \leq 2\beta ^{2}$. We also have%
\begin{eqnarray*}
\Phi ^{\prime }\left( G\right) \left[ u\right]  &=&\beta E_{G}\left[ G\right]
E_{G}\left[ u\right] -\beta E_{G}\left[ Gu\right] +E_{G}\left[ u\right]  \\
\Psi ^{\prime }\left( G\right) \left[ u\right]  &=&\beta E_{G}\left[ H_{0}%
\right] E_{G}\left[ u\right] -\beta E_{G}\left[ H_{0}u\right] .
\end{eqnarray*}%
Since $\left\Vert H_{0}\right\Vert ,\left\Vert G\right\Vert \in \mathbb{B}%
_{1}$ we have $\left\Vert \Phi ^{\prime }\right\Vert \leq 2\beta +1$ and $%
\left\Vert \Psi ^{\prime }\right\Vert \leq 2\beta $. By a somewhat tedious
computation%
\begin{eqnarray*}
\Phi ^{\prime \prime }\left[ u\right] \left[ v\right]  &=&2\beta ^{2}E_{G}%
\left[ G\right] E_{G}\left[ v\right] E_{G}\left[ u\right] -\beta ^{2}E_{G}%
\left[ G\right] E_{G}\left[ vu\right] +\beta ^{2}E_{G}\left[ Guv\right]  \\
&&-\beta ^{2}E_{G}\left[ Gv\right] E_{G}\left[ u\right] -\beta ^{2}E_{G}%
\left[ Gu\right] E_{G}\left[ v\right] -2\beta E_{G}\left[ uv\right] +2\beta
E_{G}\left[ u\right] E_{G}\left[ v\right] ,
\end{eqnarray*}%
which gives $\left\Vert \Phi ^{\prime \prime }\right\Vert \leq 6\beta
^{2}+4\beta $. Similarly, and a bit simpler, one obtains $\left\Vert \Psi
^{\prime \prime }\right\Vert \leq 6\beta ^{2}$. Adding these estimates we get%
\begin{eqnarray*}
F &=&\beta \left( \Psi -\Phi \right) -\Xi +\ln Z_{0} \\
\left\Vert F^{\prime }\right\Vert  &\leq &4\beta ^{2}+2\beta  \\
\left\Vert F^{\prime \prime }\right\Vert  &\leq &12\beta ^{3}+6\beta ^{2}.
\end{eqnarray*}%
The chain rule then gives%
\begin{eqnarray*}
M\left( F\circ f\right)  &\leq &\left( 4\beta ^{2}+2\beta \right) M\left(
f\right) \leq \frac{4\beta ^{2}+2\beta }{n} \\
J\left( F\circ f\right)  &\leq &n\left( 12\beta ^{3}+6\beta ^{2}\right)
M\left( f\right) ^{2}+0\leq \frac{12\beta ^{3}+6\beta ^{2}}{n}.
\end{eqnarray*}%

\subsection{Table of notation}

\bigskip
\begin{tabular}{|l|l|l|}
\hline
Symbol & Quick description & Section \\ 
\hline
\hline
$\mathcal{X}$ & space of observarions & \ref{Section Introduction} \\ 
$X_{i}$ & independent random variables in $\mathcal{X}$ & \ref{Section
Introduction} \\ 
$\mu _{i}$ & distribution of $X_{i}$ & \ref{Section Introduction} \\ 
$\mathbf{X}$ & random vector composed of the $X_{i}$ & \ref{Subsection
EfronStein and Bernstein} \\ 
$\mathcal{A}_{n}$ & bounded measurable functions $f:\mathcal{X}%
^{n}\rightarrow 
\mathbb{R}
$ & \ref{Subsection EfronStein and Bernstein} \\ 
$\mathbf{x}$ & vector in $\mathcal{X}^{n}$ & \ref{Subsection EfronStein and
Bernstein} \\ 
$E\left[ f\right] $ & $E\left[ f\right] =E\left[ f\left( \mathbf{X}\right) %
\right] =E\left[ f\left( X_{1},...,X_{n}\right) \right] $ for $f\in \mathcal{%
A}_{n}$ & \ref{Subsection EfronStein and Bernstein} \\ 
$\sigma ^{2}\left( f\right) $ & Variance of $f\left( X_{1},...,X_{n}\right) $
for $f\in \mathcal{A}_{n}$ & \ref{Subsection EfronStein and Bernstein} \\ 
$D_{y,y^{\prime }}^{k}$ & partial difference operator & \ref{Section
Introduction} \\ 
$S_{y}^{k}$ & substitution operator & \ref{Subsection VarianceEstimation} \\ 
$S_{-}^{k}$ & deletion operator & \ref{Subsection VarianceEstimation} \\ 
$M\left( f\right) $ & distance to constant functions & \ref{Section
Introduction} \\ 
$J\left( f\right) $ & distance to additive functions & \ref{Section
Introduction} \\ 
$\sigma _{k}^{2}\left( f\right) $ & $k$-th conditional variance & \ref%
{Subsection EfronStein and Bernstein} \\ 
$\Sigma ^{2}\left( f\right) $ & sum of conditional variances & \ref%
{Subsection EfronStein and Bernstein} \\ 
$v_{n}$ & sample variance & \ref{Section Introduction} \\ 
$v_{f}$ & variance estimator for $f\in \mathcal{A}_{n}$. Note $v_{f}\in 
\mathcal{A}_{n+1}$ & \ref{Subsection VarianceEstimation} \\ 
$K_{-},K_{+}$ & estimation error coefficients for $v_{f}$ & \ref{Subsection
VarianceEstimation} \\ 
$d_{\mathcal{N}}$ & distance to normality & \ref{Subsection Normal
Approximation}\\
\hline
\end{tabular}

\end{document}